\documentclass[10pt,twocolumn,twoside]{IEEEtran}

\usepackage[utf8]{inputenc}
\usepackage[T1]{fontenc}
\usepackage{amsmath,amsthm,amsfonts,amssymb,bm,accents,dsfont}
\usepackage{graphicx,cite,enumerate,afterpage}
\usepackage{url}

\usepackage{blindtext}

\usepackage{hyperref}

\usepackage{xcolor}

\definecolor{Set1.red}{rgb}{0.894117647058824,0.101960784313725,0.109803921568627}
\definecolor{Set1.orange}{rgb}{1,0.498039215686275,0}

\DeclareMathOperator{\E}{\mathbb{E}}

\newcommand{\norm}[1]{\ensuremath{\left\| #1 \right\|}}
\newcommand{\abs}[1]{\ensuremath{{\left\vert #1 \right\vert}}}

\DeclareMathOperator{\conv}{conv}
\DeclareMathOperator{\cconv}{\overline{conv}}

\DeclareMathOperator{\indicator}{\mathbb{I}}
\newcommand{\ceil}[1]{\left \lceil #1 \right \rceil}

\DeclareMathOperator*{\argmin}{argmin}
\DeclareMathOperator*{\argmax}{argmax}
\DeclareMathOperator*{\minimize}{minimize}

\DeclareMathOperator{\subjectto}{subject\ to}

\newcommand{\calA}{\ensuremath{\mathcal{A}}}
\newcommand{\calB}{\ensuremath{\mathcal{B}}}
\newcommand{\calC}{\ensuremath{\mathcal{C}}}
\newcommand{\calD}{\ensuremath{\mathcal{D}}}
\newcommand{\calE}{\ensuremath{\mathcal{E}}}
\newcommand{\calF}{\ensuremath{\mathcal{F}}}

\newcommand{\calH}{\ensuremath{\mathcal{H}}}
\newcommand{\calI}{\ensuremath{\mathcal{I}}}

\newcommand{\calM}{\ensuremath{\mathcal{M}}}

\newcommand{\calS}{\ensuremath{\mathcal{S}}}
\newcommand{\calT}{\ensuremath{\mathcal{T}}}

\newcommand{\calX}{\ensuremath{\mathcal{X}}}
\newcommand{\calY}{\ensuremath{\mathcal{Y}}}
\newcommand{\calZ}{\ensuremath{\mathcal{Z}}}

\newcommand{\bzero}{\ensuremath{\bm{0}}}

\newcommand{\ba}{\ensuremath{\bm{a}}}

\newcommand{\bc}{\ensuremath{\bm{c}}}

\newcommand{\bp}{\ensuremath{\bm{p}}}

\newcommand{\bs}{\ensuremath{\bm{s}}}

\newcommand{\bx}{\ensuremath{\bm{x}}}

\newcommand{\bdelta}{\ensuremath{\bm{\delta}}}

\newcommand{\bmu}{\ensuremath{\bm{\mu}}}

\newcommand{\btheta}{\ensuremath{\bm{\theta}}}

\newcommand{\bsigma}{\ensuremath{\bm{\sigma}}}

\newcommand{\bbN}{\ensuremath{\mathbb{N}}}

\newcommand{\bbR}{\ensuremath{\mathbb{R}}}

\newcommand{\fkA}{\ensuremath{\mathfrak{A}}}

\newcommand{\fkD}{\ensuremath{\mathfrak{D}}}

\newcommand{\fke}{\ensuremath{\mathfrak{e}}}

\newcommand{\fkm}{\ensuremath{\mathfrak{m}}}

\newcommand{\fkp}{\ensuremath{\mathfrak{p}}}
\newcommand{\fkq}{\ensuremath{\mathfrak{q}}}

\newcommand{\bxt}{\ensuremath{\bm{{\tilde{x}}}}}

\def\st/{\textsuperscript{st}}
\def\nd/{\textsuperscript{nd}}
\def\rd/{\textsuperscript{rd}}
\def\th/{\textsuperscript{th}}

\newcommand{\setN}{\bbN}

\newcommand{\setR}{\bbR}

\newcommand{\zeros}{\ensuremath{\bm{0}}}

\makeatletter
\def\nnil{\nil}
\newcounter{prob}
\newenvironment{prob}[1][\nil]{%
	\def\tmp{#1}
	\equation
	\ifx\tmp\nnil
		\refstepcounter{prob}
		\edef\@currentlabel{\theprob}\expandafter\ltx@label\expandafter{P\Roman{prob}-counter-number}
		\tag{P\Roman{prob}}
	\else
		\tag{\tmp}
	\fi
	\aligned%
}{%
	\endaligned\endequation%
}

\newcounter{dual}

\makeatother

\newenvironment{prob*}{%
	\csname equation*\endcsname%
	\aligned%
}{%
	\endaligned%
	\csname endequation*\endcsname%
}

\newcommand{\includesvg}[2][scale=1]{\includegraphics[#1]{#2.pdf}}

\usepackage[english]{babel}

\usepackage{algpseudocode,algorithm}
\algrenewcommand\algorithmicdo{}
\algrenewtext{EndFor}{\algorithmicend}
\algrenewtext{EndProcedure}{\algorithmicend}

\usepackage{multirow}

\usepackage{tikz}
\usetikzlibrary{arrows,arrows.meta,cd}
\usepackage{pgfplots}
\usepgfplotslibrary{groupplots}
\usepgfplotslibrary{fillbetween}

\newtheorem{theorem}{Theorem}
\newtheorem{proposition}{Proposition}[section]
\newtheorem{corollary}{Corollary}
\newtheorem{lemma}{Lemma}[section]
\theoremstyle{definition}
\newtheorem{definition}{Definition}
\newtheorem{remark}{Remark}
\newtheorem{example}{Example}
\newtheorem{assumption}{Assumption}

\newcommand{\bhtheta}{\ensuremath{\bm{{\hat{\theta}^\star}}}}
\newcommand{\bhdtheta}{\ensuremath{\hat{\btheta}^\dagger}}
\newcommand{\bhmu}{\ensuremath{\bm{{\hat{\mu}^\star}}}}

\newcommand{\bmut}{\ensuremath{\bm{{\tilde{\mu}}}_\nu^\star}}

\newcommand{\calHb}{\ensuremath{\bar{\calH}}}

\title{Constrained Learning with Non-Convex Losses}

\author{Luiz~F.~O.~Chamon, Santiago~Paternain, Miguel~Calvo-Fullana, and Alejandro~Ribeiro%
\thanks{\emph{Luiz~F.~O.~Chamon} (\mbox{\texttt{lfochamon@berkeley.edu}}) is with the Simons Institute for the Theory of Computation, University of California, Berkeley.\\
\emph{Santiago~Paternain} (\mbox{\texttt{paters@rpi.edu}}) is with the Department of Electrical, Computer and Systems Engineering,  Rensselaer Polytechnic Institute.\\
\emph{Miguel~Calvo-Fullana} (\mbox{\texttt{cfullana@mit.edu}}) is with the  Department of Aeronautics and Astronautics, Massachusetts Institute of Technology.\\
\emph{Alejandro~Ribeiro} (\mbox{\texttt{aribeiro@seas.upenn.edu}}) is with the Department of Electrical and Systems Engineering, University of Pennsylvania.}.
\thanks{Part of the results in this paper appeared in~\cite{Chamon20t, Chamon20p}.}%
}

\begin{document}
\maketitle
\begin{abstract}

Though learning has become a core component of modern information processing, there is now ample evidence that it can lead to biased, unsafe, and prejudiced systems. The need to impose requirements on learning is therefore paramount, especially as it reaches critical applications in social, industrial, and medical domains. However, the non-convexity of most modern statistical problems is only exacerbated by the introduction of constraints. Whereas good unconstrained solutions can often be learned using empirical risk minimization, even obtaining a model that satisfies statistical constraints can be challenging. All the more so, a good one. In this paper, we overcome this issue by learning in the empirical dual domain, where constrained statistical learning problems become unconstrained and deterministic. We analyze the generalization properties of this approach by bounding the empirical duality gap---i.e., the difference between our approximate, tractable solution and the solution of the original (non-convex)~statistical problem---and provide a practical constrained learning algorithm. These results establish a constrained counterpart to classical learning theory, enabling the explicit use of constraints in learning. We illustrate this theory and algorithm in rate-constrained learning applications arising in fairness and adversarial robustness.

\end{abstract}


\section{Introduction}
	\label{S:intro}

Learning is at the core of modern information systems upon which we increasingly rely to select job candidates, analyze medical data, and control ``smart'' applications~(home, grid, city). Central to this approach is the concept of empirical risk minimization~(ERM), in which a statistical~(expected value) optimization problem is replaced by its empirical~(sample average) counterpart, thus allowing it to be solved directly from data, without knowledge of the underlying distributions~\cite{Vapnik00t, Shalev-Shwartz04u, Mohri18f}. This approach is grounded on celebrated generalization results from learning~\cite{Vapnik00t, Shalev-Shwartz04u, Mohri18f} and stochastic optimization theory~\cite{Shapiro09l, Homem-de-Mello14m} showing that under mild conditions, the ERM solutions are close to their statistical analog for large enough sample sizes.

As these systems become ubiquitous, however, so does the need to constrain their behavior to tackle
fairness~\cite{Goh16s, Woodworth17l, Agarwal18a, Donini18e, Kearns18p, Zafar19f, Cotter19o},
robustness~\cite{Madry18t, Sinha18c, Zhang19t},
and safety~\cite{Garcia15a, Achiam17c, Paternain19l} problems. Left untethered, learning can lead to biased, prejudiced models prone to tampering~(e.g., adversarial examples) and unsafe behaviors~\cite{Datta15a, Kay15u, Angwin16m}. Such constraints can also be used to incorporate prior knowledge, such as smoothness or sparsity~\cite{Wahba90s, Berlinet11r, Eldar12c}, and tackle semi-supervised problems~\cite{Cour11l, Yu17m}.

In learning, requirements are often imposed using penalties, i.e., by integrating constraint violation costs in the ERM objective~(see, e.g.,~\cite{Goodfellow15e, Berk17a, Xu18a, Zhao18t, Sinha18c}). Yet, while it is straightforward to find penalties leading to feasible, optimal solutions when the ERM problem is convex, most modern parametrizations~(e.g., neural networks, NNs) lead to non-convex optimization problems. Designing good penalties then becomes a time-consuming trial-and-error process based on domain-expert knowledge that need not yield feasible solutions, all the more so good ones~\cite{Bertsekas09c}. While algorithms inspired by primal-dual methods have been applied in practice~\cite{Madry18t, Shaham18u, Zhang19t}, they are not supported by generalization guarantees. What is more, classical learning theory guarantees generalization with respect to the overall objective~(cost + penalty) and not with respect to the requirements it describes~\cite{Vapnik00t, Shalev-Shwartz04u, Mohri18f}. This issue is sometimes addressed by constructing models that explicitly \emph{embed} the desired properties~(e.g.,~\cite{Cohen16g, Weiler18l, Ruiz20i}), although the scale and opacity of modern machine learning systems typically render this approach impractical.

Since learning is often synonymous with ERM, a natural solution is to explicitly add constraints to these optimization problems. Given that requirements are often expressed as constraints in the first place, this approach guarantees that any solution satisfies the requirements without the need to tune penalty parameters. While constraints have been deployed in statistics since at least Neyman-Pearson~\cite{Neyman33i}, two roadblocks hinder their use in modern learning problems. First, typical parametrizations lead to non-convex constrained optimization problems that are often computationally harder than their unconstrained counterparts. While gradient descent can sometimes be used to approximately minimize a loss function even if it is non-convex~\cite{Ge18l, Brutzkus17g, Soltanolkotabi18t}, it does not guarantee feasibility.

Second, even if we could solve this constrained ERM problem, generalization guarantees exist only in specific contexts, e.g., for coherence constraints or rate-constrained learning~\cite{Garg01l, Luedtke08a, Pagnoncelli09s, Goh16s, Woodworth17l, Agarwal18a, Cotter19o}. And even then, results often hold for specific models, algorithms, and/or randomized solutions, e.g., \cite{Goh16s, Woodworth17l, Kearns18p, Agarwal18a, Cotter19o, Cotter19t}. Indeed, classical learning theory is concerned with unconstrained learning problems~\cite{Vapnik00t, Shalev-Shwartz04u, Mohri18f} and results for general constrained stochastic programs are often asymptotic, involving a myriad of distributional assumptions~\cite{Ahmed07c, Wang08s, Pagnoncelli09s, Luedtke08a, Hu12s, Guigues17n, Oliveira20s}. More recent guarantees rely on particular algorithms and hold only for randomized solutions.

This paper builds off the constrained statistical learning theory from~\cite{Chamon20p} to provide generalization guarantees for a large class of constrained learning problems, even those involving non-convex losses. Its main contribution is a generalization bound on the empirical duality gap of constrained learning, i.e., the difference between the optimal value of the constrained, statistical problem and that of an unconstrained, deterministic problem. This result implies a practical dual ascent algorithm that we show yields near-optimal and near-feasible solutions.

We approach these results in three steps. First, we show that, under mild conditions, functional constrained learning has zero \emph{duality gap}~(Section~\ref{S:duality_gap}). In contrast to~\cite{Chamon20t, Chamon20p} that considered only convex losses, we make use of recent results from non-convex functional optimization to account for the non-convex case~\cite{Chamon20f}. We then proceed to bound the \emph{parameterization gap}, i.e., the loss of optimality due to approximating the functional learning problem using a finite dimensional parameterization~(Section~\ref{S:parameterization_gap}). Finally, we analyze the \emph{empirical gap}, i.e., the error due to the use of samples instead of the unknown data distributions~(Section~\ref{S:empirical_gap}). The final bound~(Theorem~\ref{T:main}) depends not only on the number of samples, but also on the difficulty of the learning task both in terms of the parametrization used and how hard the constraints are to satisfy. We then show that the dual ascent method suggested by these results enables us to explicitly constrain models during training. We conclude by showcasing practical applications of constrained learning.


\section{A (Constrained) Learning Primer}
\label{S:csl}

Let~$\fkD_i$, $i = 0,\dots,m$, denote probability distributions over data pairs~$(\bx,y)$, with~$\bx \in \calX \subseteq \setR^d$ and~$y \in \calY \subseteq \setR$, and~$f_{\btheta}: \calX \to \setR^k$ be a function associated with the parameter vector~$\btheta \in \Theta \subseteq \setR^p$. We denote the hypothesis class induced by these functions~$\calH = \{f_{\btheta} \mid \btheta \in \Theta\}$. For convenience, we can interpret~$\bx$ as a feature vector or a system input, $y$ as a label or a measurement, $f_{\btheta}$ as a parametrized classifier or estimator, $\fkD_0$ as a nominal joint distribution, and the other~$\fkD_i$ as conditional distributions over which requirements are imposed. For instance, the~$\fkD_i$ can be used to represent adversarial input perturbations for robust learning~(as in Section~\ref{S:robust}) or subgroups of the population in fair learning applications~(as in Section~\ref{S:rate}). For classification problems, $\calY$ is finite, typically a subset of~$\setN$.

The constrained learning problem is defined as
\begin{prob}[\textup{P-CSL}]\label{P:csl}
	P^\star = \min_{\btheta \in \Theta}&
		&&\E_{(\bx,y) \sim \fkD_0} \!\Big[ \ell_0\big( f_{\btheta}(\bx),y \big) \Big]
	\\
	\subjectto& &&\E_{(\bx,y) \sim \fkD_i} \!\Big[ \ell_i\big( f_{\btheta}(\bx),y \big) \Big] \leq c_i
		\text{,}
	\\
	&&&\qquad i = 1,\ldots,m
		\text{,}
\end{prob}
where~$\ell_i: \setR^k \times \calY \to [0,B]$, $i = 0,\dots,m$, together with the~$c_i$, encode the performance metric and the desired statistical properties of the solution. The value~$P^\star$ of~\eqref{P:csl} belongs to the extended real line, i.e., $\setR \cup \{\infty\}$. In particular, $P^\star = \infty$ whenever the constrained learning problem is infeasible, i.e., for all~$\btheta \in \Theta$ there exists~$i$ such that~$\E_{\fkD_i} \!\Big[ \ell_i\big( f_{\btheta}(\bx),y \big) \Big] > c_i$. We omit the random variables over which expectations are taken whenever they are clear from the context. Since~$\infty$ is used only as a symbol to denote infeasibility, we let~$\infty - \infty = 0$.

Observe that~\eqref{P:csl} explicitly considers statistical constraints rather than parameter constraints, such as quadratic reqularization~($\norm{\btheta}_2 \leq c$) or sparsity~($\norm{\btheta}_1 \leq c$). The latter, embedded in~$\Theta$, are deterministic and can be directly imposed using projections. In contrast, the constraints in~\eqref{P:csl} cannot even be evaluated as they depend on unknown distributions~$\fkD_i$. The goal of constrained learning theory is to establish when and how~\eqref{P:csl} can be solved using only samples from the~$\fkD_i$ as classical learning theory does for unconstrained learning.

\subsection{PAC learning}

The unconstrained version of~\eqref{P:csl}, namely,
\begin{prob}\label{P:sl}
	Q^\star = \min_{\btheta \in \Theta}\ \E_{(\bx,y) \sim \fkD_0} \!\Big[ \ell_0\big( f_{\btheta}(\bx),y \big) \Big]
		\text{,}
\end{prob}
is at the core of celebrated Bayesian estimators, such as Kalman filters, and virtually every modern learning algorithm~\cite{Kailath00l, Friedman01t, Shalev-Shwartz04u}. Classical learning theory studies the conditions under which~\eqref{P:sl} can be solved based only on samples from~$\fkD_0$, i.e., without the knowledge of the distribution itself. When the hypothesis class~$\calH$ induced by the parametrization allows~$Q^\star$ to be approximated arbitrarily well and with arbitrarily high probability, it is said to be \emph{probably approximately correct~(PAC) learnable}~\cite{Valiant84a, Haussler92d, Vapnik00t, Shalev-Shwartz04u}.

\begin{definition}[PAC learnability]\label{D:pac}

A hypothesis class~$\calH$ is~(agnostic)%
\footnote{Throughout this work, we consider only the agnostic setting~(as opposed to the realizable one~\cite[Def.~3.1]{Shalev-Shwartz04u}). For conciseness, we therefore omit the qualifier ``agnostic'' from now on.}
\emph{probably approximately correct~(PAC)} learnable with respect to the loss~$\ell_0$ if there exists an algorithm that, for every~$\epsilon, \delta \in (0,1)$ and every distribution~$\fkD_0$, can obtain~$f_{\btheta} \in \calH$ based on~$N_\calH(\epsilon,\delta)$ samples from~$\fkD_0$ such that
\begin{equation}\label{E:pac}
	\E \big[ \ell_0\big( f_{\btheta}(\bx),y \big) \big] \leq Q^\star + \epsilon \text{ with probability } 1-\delta
		\text{.}
\end{equation}
\end{definition}

\noindent While the original definition in~\cite{Valiant84a} also requires the sample complexity~$N_\calH$ to be polynomial in~$1/\epsilon$ and~$1/\delta$, this distinction will not be important to us.

PAC learnability limits the complexity of a hypothesis class: the richer~$\calH$ is, the more samples are required to pinpoint a hypothesis that has small error on the entire distribution.  Different complexity measures exist that allow us to determine whether a hypothesis class is learnable or not. Though our derivations do not rely on a specific one, we introduce two well-known measures below to illustrate our results.

Consider the sample set~$\calS = \{(\bx_i,y_i) \in \calX \times \calY\}$, $\abs{\calS} = N$, and the set~$\calF$ of functions $f: \calX \times \calY \to \setR$. Denote by~$\calF \circ \calS$ the set of vectors in~$\setR^\abs{\calS}$ achievable by applying the functions in~$\calF$ to the samples in~$\calS$, namely,
\begin{equation}\label{E:achievable_vectors}
	\calF \circ \calS = \Big\{ \big[ f(\bx_1,y_1)\ \cdots\ f(\bx_N,y_N) \big] \mid f \in \calF \Big\}
		\text{.}
\end{equation}
Both the VC dimension~(Definition~\ref{D:vc}) and the Rademacher complexity~(Definition~\ref{D:rademacher}) characterize the learning complexity of~$\calH$ in terms of the size of the set~$(\ell \circ \calH) \circ \calS$, where~$\ell \circ \calH$ is used to denote the set of functions~$\{\ell(h(\cdot),\cdot) \mid h \in \calH\}$. The former does so in a combinatorial, worst-case sense, whereas the latter does so on average.

\begin{definition}[VC dimension~{\cite[Section~3.6]{Vapnik00t}}]\label{D:vc}
The \emph{VC dimension} of a hypothesis class~$\calH$ with respect to the loss~$\ell$ is defined as
\begin{equation}\label{E:vc_dimension}
	d_\text{VC} = \max \{m \mid \Pi(m) = 2^m\}
\end{equation}
for the \emph{growth function} $\Pi(m) = \max_{\abs{\calS} = m}\ \abs{ (\calT \circ \ell \circ \calH) \circ \calS }$, where~$\calT = \{ \indicator(\cdot - b > 0) \mid b \in (0,B) \}$ is a set of threshold functions and~$\indicator(\calE)$ denotes the indicator function that is one over the event~$\calE$ and zero otherwise. To be more specific, the set~$(\calT \circ \ell \circ \calH) \circ \calS$ contains all vectors of the form~$\big[ \indicator \big( \ell\big( f_{\btheta}(\bx_n), y_n \big) - b \big) \big]$, i.e., all possible binary sequences obtained by thresholding the losses. Note that~$\calT$ has no effect for the~0-1~loss, in which case Definition~\ref{D:vc} reduces to the typical one found in, e.g., \cite[Def.~6.5]{Shalev-Shwartz04u} or~\cite[Def.~3.10]{Mohri18f}.

\end{definition}

\begin{definition}[Rademacher complexity~{\cite{Bartlett02r}\cite[Def.~3.1--3.2]{Mohri18f}}]\label{D:rademacher}
The \emph{Rademacher complexity} of a hypothesis class~$\calH$ with respect to the loss~$\ell$ and the distribution~$\fkD$ is defined as
\begin{equation}\label{E:rademacher_complexity}
	R_N = \E_{\calS \sim \fkD^N} \!\Big[ \hat{R}\big( (\ell \circ \calH) \circ \calS \big) \Big]
\end{equation}
for the \emph{empirical Rademacher complexity}~$\hat{R}$ defined for any set of vectors~$\calA \subset \setR^N$ as
\begin{equation}\label{E:rademacher_complexity_empirical}
	\hat{R}(\calA) = \E_{\bsigma} \left[ \sup_{\ba \in \calA}\ \frac{1}{N} \sum_{n = 1}^N \sigma_n a_n \right]
		\text{,}
\end{equation}
where~$\bsigma \in \setR^N$ is a random vector whose elements~$\sigma_n$ are drawn i.i.d.\ according to~$\Pr[\sigma_n = +1] = \Pr[\sigma_n = -1] = 1/2$.
\end{definition}

A fundamental result in learning theory states that for binary classification problems, $\calH$ is PAC learnable with respect to the~$0/1$-loss~$\indicator(f_{\btheta}(\bx) \neq y)$ if and only if it has finite VC dimension. In this case, ERM is a PAC learner, i.e.,
\begin{prob}[\textup{P-ERM}]\label{P:erm}
	\hat{Q}^\star = \min_{\btheta \in \Theta}\ \frac{1}{N} \sum_{n = 1}^{N}
		\ell_0\big( f_{\btheta}(\bx_{n}), y_{n} \big)
\end{prob}
for independent samples~$(\bx_n,y_n) \sim \fkD_0$ yields a PAC solution of~\eqref{P:sl}~(see, e.g., \cite[Thm.~6.7]{Shalev-Shwartz04u}). Though this equivalence does not hold in general, Definitions~\ref{D:vc} and~\ref{D:rademacher} can still be used to bound the sample complexity~$N_\calH$ of other learning tasks~(e.g., see~Proposition~\ref{T:emp_bounds}). Note that the Rademacher complexity depends on the distribution of the data while PAC learnability does not. Still, it is often the case that distribution-independent upper bounds can be found for~$R_N$~\cite{Shalev-Shwartz04u, Mohri18f}.

In constrained learning, however, it is not enough to approximate the value~$P^\star$ of~\eqref{P:csl}, since a solution must also satisfy its constraints. In fact, feasibility often takes priority over performance in constrained problems: regardless of how accurate a \emph{fair} classifier is, it serves no practical purpose unless it meets the fairness requirements~(see Section~\ref{S:sims} for an example). Hence, PAC learning is not sufficiently strict to address the problem of learning under requirements. In the sequel, we summarize the constrained learning framework introduced in~\cite{Chamon20p}.

\subsection{Probably approximately correct constrained learning}
\label{S:pacc}

Since we do not have access to the distributions~$\fkD_i$ required to evaluate~\eqref{P:csl}, we cannot expect to obtain an exact solution and must settle for one that is \emph{good enough}. Similar to classical learning theory, we next establish what is considered ``good enough'' for a constrained learning problem.

\begin{definition}[Near-PACC learnability~{\cite[Def.~2]{Chamon20p}}]\label{D:pacc}

A hypothesis class~$\calH$ is \emph{nearly probably approximately correct constrained~(near-PACC)} learnable with respect to~$\{\ell_0, (\ell_i,c_i)\}$, if there exists~$\epsilon_0 \geq 0$ and an algorithm that, for every~$\epsilon, \delta \in (0,1)$ and every distribution~$\fkD_i$, $i = 0,\dots,m$, can obtain~$f_{\btheta} \in \calH$ using~$N_\calH(\epsilon,\delta,m)$ samples from each~$\fkD_i$ that is, with probability $1-\delta$,
\begin{enumerate}[1)]
	\item probably approximately optimal, i.e.,
	\begin{equation}\label{E:pacc_optimal}
		\abs{\E_{(\bx,y) \sim \fkD_0} \!\big[ \ell_0\big( f_{\btheta}(\bx),y \big) \big] - P^\star} \leq \epsilon_0 + \epsilon
			\text{, and}
	\end{equation}

	\item probably approximately feasible, i.e.,
	\begin{equation}\label{E:pacc_feasible}
		\E_{(\bx,y) \sim \fkD_i} \!\big[ \ell_i\big( f_{\btheta}(\bx),y \big) \big]
			\leq c_i + \epsilon
			\text{,} \quad \text{for all } i \geq 1
			\text{.}
	\end{equation}
\end{enumerate}

\end{definition}

Definition~\ref{D:pacc} is an extension of the PAC framework from classical learning theory to the problem of learning under requirements. Indeed, for~$\epsilon_0 = 0$, \eqref{E:pacc_optimal} is the classical definition of PAC learnability~(see Definition~\ref{D:pac}). In fact, PACC learnability implies PAC learnability~(Remark~\ref{R:pac}). This is, however, not enough to enable constrained learning since a PAC~$f_{\btheta}$ may not be feasible for~\eqref{P:csl}. Hence, a PACC learner must also satisfy the approximate feasibility condition~\eqref{E:pacc_feasible}. The additional ``C'' in PACC serves to remind ourselves of this fact.

Another important distinction with PAC learning is the presence of a fixed tolerance~$\epsilon_0$. Notice that this tolerance is independent of the distributions and affects only the value of the problem, i.e., it does not interfere with either the sample complexity~$N_\calH$ or the constraint satisfaction~\eqref{E:pacc_feasible}. Instead, it characterizes an intrinsic limitation of the learning task related to the \emph{approximation error} found in classical~(unconstrained) learning~\cite{Shalev-Shwartz04u}. However, it is now coupled to the learning problem by the constraints and can no longer be treated as a separate source of error. Hence, near-PACC learnability is in fact a hierarchy: when~$\epsilon_0 \geq B$, \eqref{E:pacc_optimal} holds trivially and near-PACC reduces to a feasibility learning problem. We are therefore interested in the smallest~$\epsilon_0$ for which Definition~\ref{D:pacc} holds and when it occurs for~$\epsilon_0 = 0$, we simply say that~$\calH$ is \emph{PACC learnable}.

Finally, observe that the sample complexity~$N_\calH$ may now depend on the number of constraints~$m$. In fact, it often does~(Theorem~\ref{T:main}). This dependency precludes the formulation of pathological learning problems that could be described using an exponential number of constraints.

\begin{remark}\label{R:pac}
It is easy to see that if the hypothesis class~$\calH$ is PACC learnable with respect to~$\{\ell_0,(\ell_i,B)\}$, then it is PAC learnable with respect to~$\ell_0$. Indeed, consider the constrained learning problem
\begin{prob}\label{P:pacc_pac}
	P_0^\star = \min_{\btheta \in \Theta}&
		&&\E_{\fkD_0} \!\Big[ \ell_0\big( f_{\btheta}(\bx),y \big) \Big]
	\\
	\subjectto& &&\E_{\fkD_0} \!\Big[ \ell_i\big( f_{\btheta}(\bx),y \big) \Big] \leq B
		\text{,}
	\\
	&&&\quad i = 1,\ldots,m
		\text{.}
\end{prob}
Since the losses are~$B$-bounded, the feasibility set of~\eqref{P:pacc_pac} is~$\Theta$. Hence, \eqref{P:pacc_pac} has the same value as the unconstrained learning problem~\eqref{P:sl}, i.e., $P_0^\star = Q^\star$. Given the hypothesis class~$\calH$ is PACC learnable~($\epsilon_0 = 0$) with respect to~$\{\ell_0,(\ell_i,B)\}$, there exists an algorithm that can obtain~$f_{\btheta} \in \calH$ such that~$\abs{\E_{\fkD_0} \!\big[ \ell_0\big( f_{\btheta}(\bx),y \big) \big] - Q^\star} \leq \epsilon$ from~$\bar{N}(\epsilon,\delta,m)$ samples for any~$\epsilon,\delta \in (0,1)$. Hence, $\calH$ is also PAC learnable~(Definition~\ref{D:pac}).
\end{remark}

\subsection{Empirical constrained risk minimization}

Despite its similarities to PAC learning~(e.g., see Remark~\ref{R:pac}), PACC learning has strikingly different behaviors. In particular, while PAC learnability is often equivalent to ERM learnability~\cite[Thm.~6.7]{Shalev-Shwartz04u}, this is not the case for constrained learning. Said otherwise, while~\eqref{P:erm} is typically a PAC learner, its constrained counterpart is generally not.

Indeed, consider the empirical constrained risk minimization~(ECRM) problem
\begin{prob}[\textup{P-ECRM}]\label{P:ecrm}
	\hat{P}^\star = \min_{\btheta \in \Theta}&
		&&\frac{1}{N_0} \sum_{n_0 = 1}^{N_0} \ell_0\big( f_{\btheta}(\bx_{n_0}), y_{n_0} \big)
	\\
	\subjectto& &&\frac{1}{N_i} \sum_{n_i = 1}^{N_i} \ell_i\big( f_{\btheta}(\bx_{n_i}), y_{n_i} \big)
		\leq c_i
		\text{,}
	\\
	&&&\qquad i = 1,\ldots,m
		\text{,}
\end{prob}
which approximates the expectations in~\eqref{P:csl} using~$N_i$ samples~$(\bx_{n_i},y_{n_i}) \sim \fkD_i$. The following example shows that~\eqref{P:ecrm} can be almost surely wrong, even for a PAC learnable hypothesis class.

\begin{example}\label{R:ecrm}
Consider the learning problem
\begin{prob}\label{P:example_ecrm}
	P_e^\star = \min_{\btheta \in \Theta}&
		&&J(\btheta) \triangleq \E_{\fkD_0} \!\big[ \vert y \btheta^\top \bx \vert \big]
	\\
	\subjectto& &&\E_{\fkD_1} \!\big[ y \btheta^\top \bx \big] \leq -1
		\text{,} \quad
	\E_{\fkD_2} \!\big[ y \btheta^\top \bx \big] \leq 1
		\text{,}
\end{prob}
where~$\fkD_0$ is the distribution of
\begin{equation*}
	(\bx,y) = \begin{cases}
		([\tau,-\tau],-1) \text{,} &\text{with prob.\ } 1/2
		\\
		([0,\alpha],1) \text{,} &\text{with prob.\ } 1/2
	\end{cases}
		\text{,}
\end{equation*}
$\fkD_1$ is such that~$(\bx,y) = ([-1,\tau],1)$, and~$\fkD_2$ is such that~$(\bx,y) = ([-\tau,1],1)$, where~$\alpha$ is drawn uniformly at random from~$[0,1/4]$ and~$\tau$ is drawn uniformly at random from~$[-1/2,1/2]$. Notice that the~$\fkD_i$ are therefore correlated through the random variable~$\tau$. The hypothesis class is induced by the finite set~$\Theta = \{[1,1]; [1,0]\}$. Notice that under these distributions, the constraints in~\eqref{P:example_ecrm} reduce to~$-\theta_1 \leq -1$ and~$\theta_2 \leq 1$. Hence, the statistical~\eqref{P:example_ecrm} is effectively unconstrained and its optimal value is~$P_e^\star = 1/16$ since
\begin{equation}\label{E:example_solutions}
	J(\btheta) = \begin{cases}
		1/16 \text{,} &\btheta = [1,1]
		\\
		1/8 \text{,} &\btheta = [1,0]
	\end{cases}
		\text{.}
\end{equation}

For its empirical version, however, the constraints can be written as
\begin{equation*}
	\theta_1 \geq 1 + \bar{\tau} \theta_2
	\quad \text{and} \quad
	\theta_2 \leq 1 + \bar{\tau} \theta_1
		\text{,}
\end{equation*}
where~$\bar{\tau} = \dfrac{1}{N} \sum_{n = 1}^N \tau_{n}$ is the empirical average of i.i.d.\ samples~$\tau_{n}$ drawn uniformly at random from~$[-1/2,1/2]$. Whenever~$\bar{\tau} > 0$, the first constraint implies~$\theta_2 = 0$. Hence, from~\eqref{E:example_solutions}, the statistical objective evaluates to~$J(\bhtheta) = 1/8$. Similarly, if~$\bar{\tau} < 0$, then the second constraint implies~$\theta_2 = 0$, which again yields a population value of~$J(\bhtheta) = 1/8$. Given that~$\tau$ is a continuous distribution, we immediately obtain that
\begin{equation*}
	\Pr\big[\big\vert J(\bhtheta) - P_e^\star \big\vert \leq 1/32 \big] = \Pr[\bar{\tau} = 0] = 0
		\text{.}
\end{equation*}
On the other hand, the unconstrained learning problem~$Q_e^\star = \min_{\btheta \in \Theta}\ \E_{\fkD_0} \!\big[ \vert y \btheta^\top \bx \vert \big]$ can be solved using ERM and~$\Pr[J(\bhtheta) \leq Q_e^\star + \epsilon] \to 0$ as~$N \to \infty$ for all~$\epsilon > 0$ by the law of large numbers.
\end{example}

Example~\ref{R:ecrm} shows that~\eqref{P:ecrm} may not be a PACC learner even when the hypothesis class is PAC/ERM learnable. This occurs because the requirements in~\eqref{P:example_ecrm} are so sensitive that they modify the feasibility set for almost every realization of its empirical version. For~\eqref{P:example_ecrm}, ECRM turns out to be a near-PACC, although with~$\epsilon_0 = 3/8 \gg P_e^\star$. This again reinforces the importance of~$\epsilon_0$ to be small for a near-PACC learner to be useful. Additionally, if the requirements are stringent, the empirical~\eqref{P:ecrm} may be infeasible even though the original~\eqref{P:csl} has a feasible solution, in which case the difference between their values is unbounded.

Example~\ref{R:ecrm} suggests that we may overcome these issues by relaxing~\eqref{P:ecrm}, i.e., by replace its constraints with
\begin{equation}\label{E:ecrm_relaxation}
	\frac{1}{N_i} \sum_{n_i = 1}^{N_i} \ell_i\big( f_{\btheta}(\bx_{n_i}), y_{n_i} \big) \leq c_i + \xi
		\text{,}
\end{equation}
where~$\xi > 0$ is an estimate of the empirical approximation error that guarantees the feasibility set of~\eqref{P:csl} is included in that of~\eqref{P:ecrm} with high probability. Yet, while~\eqref{E:ecrm_relaxation} addresses the issue of feasibility in Definition~\ref{D:pacc}, it is not clear how this relaxation affects the value the ECRM solution. Indeed, the feasibility set of the relaxed~\eqref{P:ecrm} is likely larger than that of~\eqref{P:csl}, allowing hypotheses with potentially lower objective value that are excluded by the original problem. This can lead to estimates of~$P^\star$ that violate the two-sided bound in~\eqref{E:pacc_optimal}. For convex problems, perturbation results can be used to connect the value of~$\xi$ to the magnitude of the deviation from~$P^\star$~\cite{Bonnans00p, Bertsekas09c}. For most modern machine learning models, however, \eqref{P:csl} is non-convex even if the losses themselves are convex. These issues are only exacerbated by the fact that it is rarely possible to obtain tight estimates for~$\xi$.

The non-convexity of~\eqref{P:ecrm} also raises computational concerns. While unconstrained learning faces a similar issue, it is exacerbated here by the presence of constraints. Indeed, though it may be possible to find good approximate minimizers of~$\ell_0$ using, e.g., gradient descent~\cite{Zhang17u, Ge18l, Brutzkus17g, Soltanolkotabi18t}, even obtaining a feasible~$\btheta$ for~\eqref{P:ecrm} may be challenging. Penalty-based formulations that incorporate a fixed linear combination of the constraints the objective of an unconstrained problem are often used to sidestep these issues~\cite{Goodfellow15e, Berk17a, Xu18a, Zhao18t, Sinha18c}. However, classical learning theory only guarantees generalization for the overall value of the objective and not for requirements it describes. In fact, there may not even be a set of weights~(regularization parameters) that yields a solution of~\eqref{P:ecrm}, leading to infeasible results or unacceptably poor performance~\cite{Bertsekas09c}.

In the sequel, we put forward an alternative learning rule based on empirical duality and show that it is a~(near-)PACC learner, despite the non-convexity of~\eqref{P:csl}. Doing so, we derive mild conditions under which PACC learning is not considerably harder than PAC learning. Another advantage of this learning rule is that it involves solving only unconstrained learning problems, leading to a more practical constrained learning algorithm than ECRM~(Section~\ref{S:algorithm}).


\section{Empirical Dual Learning}
\label{S:empirical_dual}

In this section, we overcome the shortcomings of~\eqref{P:ecrm} by analyzing the gap between~\eqref{P:csl} and its empirical dual problem. Our goal is to quantify the loss of optimality incurred by replacing the constrained, statistical problem~\eqref{P:csl} by an unconstrained, empirical one.

Explicitly, define the~\emph{empirical Lagrangian} of~\eqref{P:csl} as
\begin{equation}\label{E:empirical_lagrangian}
\begin{aligned}
	\hat{L}(\btheta, \bmu) &=
		\frac{1}{N_0} \sum_{n_0 = 1}^{N_0} \ell_0\big( f_{\btheta}(\bx_{n_0}), y_{n_0} \big)
		\\
		{}&+ \sum_{i = 1}^m \mu_i \left[ \frac{1}{N_i} \sum_{n_i = 1}^{N_i}
			\ell_i\big( f_{\btheta}(\bx_{n_i}), y_{n_i} \big) - c_i \right]
		\text{,}
\end{aligned}
\end{equation}
based on~$N_i$ samples~$(\bx_{n_i},y_{n_i}) \sim \fkD_i$, where~$\bmu \in \setR^m_+$ collects the dual variables~$\mu_i$ relative to each constraint and~$\setR_+$ denotes the set of non-negative real numbers. Defining the \emph{empirical dual function} associated with~\eqref{E:empirical_lagrangian} as
\begin{equation}\label{E:empirical_dual_function}
	\hat{d}(\bmu) = \min_{\btheta \in \Theta}\ \hat{L}(\btheta, \bmu)
		\text{,}
\end{equation}
the \emph{empirical dual problem} of~\eqref{P:csl} is written as
\begin{equation}\label{P:empirical_dual}
	\hat{D}^\star = \max_{\bmu \in \hat{\calM}}\ \hat{d}(\bmu)
		\text{,}
		\tag{$\widehat{\textup{D}}$\textup{-CSL}}
\end{equation}
where~$\hat{\calM} = \{ \bmu \in \setR_+^m \mid \hat{d}(\bmu) > -\infty \}$ is the domain of~$\hat{d}$. As with~\eqref{P:csl}, $\hat{D}^\star$ takes values on the extended real line~$\setR \cup \{\infty\}$ and~$\hat{D}^\star = \infty$ whenever~\eqref{P:ecrm} is infeasible, i.e., for all~$\btheta \in \Theta$ there exists~$i$ such that~$\frac{1}{N_i} \sum_{n_i = 1}^{N_i} \ell_i\big( f_{\btheta}(\bx_{n_i}), y_{n_i} \big) > c_i$.

There are two ways of viewing~\eqref{P:empirical_dual}. The first is to consider it as the dual problem of~\eqref{P:ecrm}~(see Figure~\ref{F:proof}). In other words, to consider the dual of the empirical counterpart of~\eqref{P:csl}. However, due to the non-convex nature of these problems, it is hard to relate their value beyond the fact that~\eqref{E:empirical_dual_function} is a relaxation of~\eqref{P:ecrm}, so that~$\hat{D}^\star \leq \hat{P}^\star$. An alternative view that will turn out to be more fruitful is to consider~\eqref{P:empirical_dual} as the empirical counterpart of the dual problem of~\eqref{P:csl}, namely,
\begin{prob}[\textup{D-CSL}]\label{P:param_dual_csl}
	D^\star = \max_{\bmu \in \calM}\ d(\bmu)
\end{prob}
solved over the domain~$\calM$ of the dual function
\begin{equation}\label{E:param_dual_function_csl}
	d(\bmu) = \min_{\btheta \in \Theta}\ L(\btheta, \bmu)
\end{equation}
for the Lagrangian
\begin{equation}\label{E:lagrangian_csl_param}
\begin{aligned}
	L(\btheta, \bmu) &= \E_{(\bx,y) \sim \fkD_0} \!\Big[ \ell_0\big( f_{\btheta}(\bx),y \big) \Big]
	\\
	{}&+ \sum_{i = 1}^m \mu_i \Big(
			\E_{(\bx,y) \sim \fkD_i} \!\Big[ \ell_i\big( f_{\btheta}(\bx),y \big) \Big] - c_i
		\Big)
		\text{.}
\end{aligned}
\end{equation}
Figure~\ref{F:proof} provides an overview of the optimization problems defined in this paper and their relation.

Observe that while the empirical Lagrangian~\eqref{E:empirical_lagrangian} has a form reminiscent of the regularized formulations often used to tackle learning under requirements, the weights~(dual variables)~$\bmu$ are optimization variables in~\eqref{P:empirical_dual} rather than constants~(adjusted by trial and error or cross-validation). Also note that while~\eqref{P:empirical_dual} is exactly the dual problem of~\eqref{P:ecrm}, their non-convex nature implies that~$\hat{D}^\star$ is a lower bound on~$\hat{P}^\star$~(weak duality~\cite{Bertsekas09c}), but not necessarily a tight one. This is therefore not enough to establish the near-optimality~\eqref{E:pacc_optimal} required by PACC learnability.

Nevertheless, the main result of this section~(Theorem~\ref{T:main}) establishes that \eqref{P:empirical_dual} is indeed a near-PACC learner under the following assumptions:

\begin{assumption}\label{A:losses}
	The losses~$\ell_i(\cdot,y)$, $i = 0,\dots,m$, are $M$-Lipschitz continuous functions for all~$y \in \calY$.
\end{assumption}

\begin{assumption}\label{A:empirical}
	For~$i = 0,\dots,m$, there exists~$\zeta_i(N,\delta) \geq 0$ monotonically decreasing with~$N$ such that
	\begin{equation}\label{E:emp_approximation}
		\abs{\E_{\fkD_i} \!\big[ \ell_i(f_{\btheta}(\bx), y) \big]
			- \frac{1}{N} \sum_{n = 1}^{N} \ell_i(f_{\btheta}(\bx_n), y_n)} \leq \zeta_i(N,\delta)
			\text{,}
	\end{equation}
	for all~$\btheta \in \Theta$, with probability~$1-\delta$ over independent draws~$(\bx_n,y_n) \sim \fkD_i$.
\end{assumption}

\begin{assumption}\label{A:parametrization}
	There exists~$\nu \geq 0$ such that for each~$\phi \in \calHb = \cconv(\calH)$, the closed convex hull of~$\calH$, there exists a~$\btheta \in \Theta$ for which
	\begin{equation}\label{E:param_approximation}
		\E_{\fkD_i} \!\big[ \abs{\phi(\bx) - f_{\btheta}(\bx)} \big] \leq \nu
			\text{.}
	\end{equation}
	The closure is taken with respect to the total variation measures~\eqref{E:param_approximation} induced by the distributions~$\fkD_i$.
\end{assumption}

\begin{assumption}\label{A:slater}
	There exist~$\bm{{\theta^\prime}}, \bm{{\hat{\theta}^\prime}} \in \Theta$ such that~$f_{\bm{{\theta^\prime}}}$ and~$f_{\bm{{\hat{\theta}^\prime}}}$ are strictly feasible for~\eqref{P:csl} and~\eqref{P:ecrm} respectively, i.e., such that, for all~$i = 1,\dots,m$,
	\begin{align*}
		\E_{\fkD_i} \!\big[ \ell_i\big( f_{\bm{{\theta^\prime}}}(\bx),y \big) \big] &\leq c_i - M \nu - \xi
			\text{,}
		\\
		\frac{1}{N_i} \sum_{n_i = 1}^{N_i} \ell_i\big( f_{\bm{{\hat{\theta}^\prime}}}(\bx_{n_i}), y_{n_i} \big) &\leq c_i - \xi
			\text{,}
	\end{align*}
	with~$M$ as in Assumption~\ref{A:losses}, $\nu$ as in~\eqref{E:param_approximation}, and~$\xi > 0$.
\end{assumption}

Assumption~\ref{A:empirical} is known in learning theory as \emph{uniform convergence}. It is often used to prove PAC learnability, though they are not equivalent: \eqref{E:emp_approximation} is sufficient, but in general not necessary, for PAC learnability~\cite{Shalev-Shwartz04u}. While it may appear strict, it can be replaced by, e.g., a bound on the VC dimension or Rademacher complexity~(Definitions~\ref{D:vc}--\ref{D:rademacher}).

\begin{proposition}\label{T:emp_bounds}
	Let~$d_\text{VC}$ and~$R_{N}$ be upper bounds on the VC dimension and Rademacher complexity of~$\calH$ with respect to~$\ell_i$ respectively. Then, \eqref{E:emp_approximation} holds with
	\begin{subequations}\label{E:empirical_bounds}
	\begin{align}
		\zeta_i(N, \delta) &= B \sqrt{\frac{1}{N} \left[ 1 + \log\left( \frac{4(2N)^{d_\text{VC}}}{\delta} \right) \right]}
			\label{E:empirical_bounds_vc}
			\quad \text{or}
		\\
		\zeta_i(N, \delta) &= 2 B R_N + B \sqrt{\frac{\log(1/\delta)}{2 N}}
			\label{E:empirical_bounds_rademacher}
			\text{.}
	\end{align}
	\end{subequations}
\end{proposition}

\begin{proof}
See~\cite[eq.~(3.26)]{Vapnik00t} for~\eqref{E:empirical_bounds_vc} using the fact that~$\log(x) \geq 1-1/x$. While guarantees based on Rademacher complexity are typically stated as one-sided bounds~\cite[Thm.~3.3]{Mohri18f}, their proof are based on symmetrization arguments and can therefore be extended to yield the two-sided bound obtained from~\eqref{E:emp_approximation} and~\eqref{E:empirical_bounds_rademacher}.
\end{proof}

Whereas Assumption~\ref{A:empirical} limits the complexity of the parametrization, Assumption~\ref{A:parametrization} requires that it still be sufficiently rich, in the sense that it is a fine cover of its convex hull or equivalently, of the underlying function space it parametrizes. This occurs, for instance, when~$f_{\btheta}$ is a neural network~(parametrizing the space of continuous functions, see, e.g., \cite{Hornik89m}) or a finite linear combinations of kernels~(parametrizing a reproducing kernel Hilbert space, RKHS~\cite{Berlinet11r}). In both cases, the parametrizations satisfy a uniform approximation condition that is stronger than the total variation requirement in~\eqref{E:param_approximation}. Assumption~\ref{A:slater} guarantees that the constrained problems~\eqref{P:csl} and~\eqref{P:ecrm} are feasible and that their dual problems are well-posed. Observe, once again, that the losses~$\ell_i$ need not be convex.

The main result of this section is collected in Theorem~\ref{T:main}. For clarity, it focuses on the classification setting, i.e., finite~$\calY$. The regression case is considered in Appendix~\ref{X:zdg_regression}. In what follows, we say a measure~$\fkm$ is \emph{non-atomic} if it does not contain Dirac deltas, i.e., if for every measurable set~$\calX$ of positive measure~($\fkm(\calX) > 0$) there exists a measurable~$\calY \subset \calX$ such that~$\fkm(\calX) > \fkm(\calY) > 0$. Additionally, we say the function space~$\calF$ is \emph{decomposable} if for every~$\phi,\phi^\prime \in \calF$ and measurable set~$\calZ$, it holds that~$\bar{\phi} \in \calF$ for
\begin{equation*}
	\bar{\phi}(\bx) =
	\begin{cases}
		\phi(\bx) \text{,} &\bx \in \calZ
		\\
		\phi^\prime(\bx) \text{,} &\bx \notin \calZ
	\end{cases}
		\text{.}
\end{equation*}
Lebesgue spaces~(e.g., $L_2$ or~$L_\infty$) or more generally Orlicz spaces are typical examples of decomposable function spaces~\cite{Rockafellar98v}.

We further introduce a functional version of~\eqref{P:csl}, namely,
\begin{prob}[$\widetilde{\textup{P}}$\textup{-CSL}]\label{P:csl_variational}
	\tilde{P}^\star = \min_{\phi \in \calHb}&
		&&\E_{(\bx,y) \sim \fkD_0} \!\Big[ \ell_0\big( \phi(\bx), y \big) \Big]
	\\
	\subjectto& &&\E_{(\bx,y) \sim \fkD_i} \!\Big[ \ell_i\big( \phi(\bx), y \big) \Big]
		\leq c_i
		\text{,}
	\\
	&&&\qquad i = 1,\ldots,m
		\text{,}
\end{prob}
where~$\calHb = \cconv(\calH)$ denotes the closed convex hull of the hypothesis class~$\calH$ induced by the parametrization~$f_{\btheta}$~(as in Assumption~\ref{A:parametrization}). Its dual problem is defined as
\begin{prob}[$\widetilde{\textup{D}}$\textup{-CSL}]\label{P:dual_csl_variational}
	\tilde{D}^\star
	   = \max_{\bmu \in \tilde{\calM}}\
	        \min_{\phi \in \calHb}\ \tilde{L}\left( \phi, \bmu \right)
	           = \max_{\bmu \in \tilde{\calM}}\
	                \tilde{d}(\bmu)
		\text{,}
\end{prob}
where~$\tilde{\calM}$ is the domain of~$\tilde{d}(\bmu) = \min_{\phi \in \calHb} \tilde{L}\left( \phi, \bmu \right)$, for the Lagrangian
\begin{equation}\label{E:lagrangian_csl_variational}
\begin{aligned}
	\tilde{L}(\phi, \bmu) &= \E_{\fkD_0} \!\Big[ \ell_0\big( \phi(\bx),y \big) \Big]
	\\
	{}&+ \sum_{i = 1}^m \mu_i \Big(
			\E_{\fkD_i} \!\Big[ \ell_i\big( \phi(\bx),y \big) \Big] - c_i
		\Big)
		\text{.}
\end{aligned}
\end{equation}
Notice that~$L(\btheta,\bmu) = \tilde{L}(f_{\btheta}, \bmu)$ for the Lagrangian of~\eqref{P:csl} in~\eqref{E:lagrangian_csl_param}.

\begin{figure}[t]
\centering
\includesvg{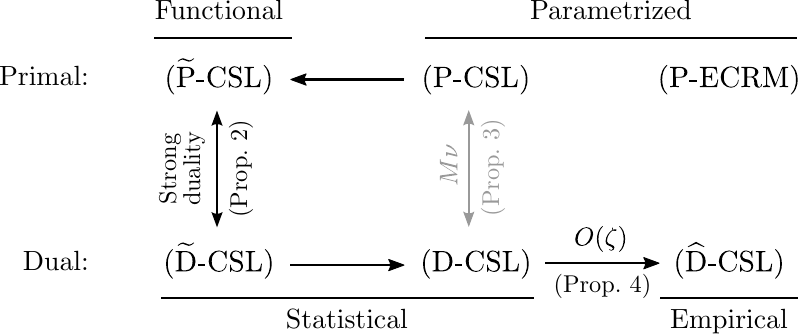}
\caption{Overview of the relation the learning problems studied in this work.}
	\label{F:proof}
\end{figure}

\begin{theorem}\label{T:main}
Let either
\begin{enumerate}[(a)]
	\item the function~$\ell_i$ be convex for~$i = 0,\dots,m$; or
	\item $\calY$ be finite, the conditional random variables~$\bx \vert y$ induced by the~$\fkD_i$ be non-atomic, and~$\cconv(\calH)$, the closed convex hull of~$\calH$, be decomposable.
\end{enumerate}
Let~$\bhmu$ be a solution of the dual problem~\eqref{P:empirical_dual} with finite~$\hat{D}^\star$. Under Assumptions~\ref{A:losses}--\ref{A:slater}, there exists~$\bhtheta \in \argmin_{\btheta \in \Theta}\ \hat{L}(\btheta,\bhmu)$ such that, with probability~$1-(3m+2)\delta$,
\begin{subequations}\label{E:empirical_duality_gap}
\begin{gather}
	\abs{P^\star - \hat{D}^\star} \leq (1 + \Delta) (M \nu + \overline{\zeta})
		\quad \text{and}
		\label{E:empirical_duality_gap_value}
	\\
	\E_{\fkD_i} \!\Big[ \ell_i\big( f_{\bhtheta}(\bx),y \big) \Big]
		\leq c_i + \zeta_{i}(N_i, \delta)
		\text{,} \quad \text{for all } i
		\text{,}
		\label{E:empirical_duality_gap_feas}
\end{gather}
\end{subequations}
where~$P^\star$ is the value of~\eqref{P:csl}, $\overline{\zeta} = \max_i \zeta_i(N_i, \delta)$, and~$\Delta = \max (\norm{\bmu^\star}_1, \norm{\bhmu}_1, \norm{\bmut}_1) \leq C$ for some constant~$C < \infty$, where~$\bmu^\star$ is a solution of~\eqref{P:param_dual_csl} and~$\bmut \in \setR^m_+$ maximizes~$\tilde{d}_{\nu}(\bmu) = \tilde{d}(\bmu) +  M\nu \norm{\bmu}_1$ for $\tilde{d}(\bmu)$ as in~\eqref{P:dual_csl_variational}.
\end{theorem}

We postpone the proof of Theorem~\ref{T:main} to the next sections~(Sections~\ref{S:duality_gap}--\ref{S:empirical_gap}) to discuss its results. Theorem~\ref{T:main} provides joint optimality--feasibility generalization guarantees for solutions of the empirical dual problem~\eqref{P:empirical_dual}~(as long as they exist). In particular, it implies that near-PACC solutions of the constrained learning problem~\eqref{P:csl} can be obtained using~\eqref{P:empirical_dual}, an unconstrained optimization program.

\begin{corollary}
	Let the VC dimension or Rademacher complexity of~$\calH$ with respect to each~$\ell_i$ be finite for all distributions~$\fkD_i$. Then, $\calH$ is near-PACC learnable with respect to~$\{\ell_0, (\ell_i,c_i)\}$ for~$\epsilon_0 = O(M \nu)$.
\end{corollary}

\noindent Note that these results hold even for non-convex losses under some mild conditions on the distributions and the hypothesis class. Theorem~\ref{T:main}, however, does not show \emph{how} to obtain solutions of~\eqref{P:empirical_dual}. We address this point in Section~\ref{S:algorithm}.

The guarantees in Theorem~\ref{T:main} are dictated by three factors: (a)~the sample size, (b)~the difficulty of the learning problem, and (c)~the richness of the parametrization.

\begin{enumerate}[(a)]
	\item \emph{Sample size.} The estimation errors~$\zeta_i$ decrease as the sample size~$N$ increases~(Assumption~\ref{A:empirical}). In fact, if the complexity of the hypothesis class is bounded in the sense of Proposition~\ref{T:emp_bounds}, then they decrease at the classical~$1/\sqrt{N}$ rate. This has a direct impact on both the near-optimality and approximate feasibility of the problem. In fact, note that if the~$\fkD_i$ are conditional distributions of~$\fkD_0$, $N_i$ can be considerably smaller than~$N_0$, jeopardizing our ability to impose requirements. This is particularly critical for classes that are minority in the sample set~(see, e.g., the fairness example in Section~\ref{S:sims}). The estimation errors also depend on the probability of success, which decreases with the number of constraints~$m$. This effect is often negligible since generalization bounds are typically logarithmic in~$\delta$~(see Proposition~\ref{T:emp_bounds}).

	\item \emph{Difficulty of the learning problem.} By difficulty, we mean the sensitivity of the learning problem to its constraints. This is embodied by the well-known sensitivity interpretation of the dual variables~\cite{Bonnans00p}, which can be formalized here due to the lack of duality gap of~\eqref{P:csl_variational}~(see Proposition~\ref{T:zdg}). Thus, $\Delta$ in~\eqref{E:empirical_duality_gap_value} effectively quantifies how stringent the constraints are for the learning problem in terms of how much performance could be gained by relaxing them. Notice that it only affects the value of the problem, illustrating the priority of requirements over cost in constrained learning.

	\item \emph{Richness of the parametrization.} The remaining source of error is the approximation capability~$\nu$ of the parametrization. It is worth noting that richer parametrizations~(smaller~$\nu$) typically involve more parameters, which in turn increases the complexity~(e.g., VC dimension, Rademacher complexity) of their hypothesis class and, consequently, $\zeta_i$~(see Proposition~\ref{T:emp_bounds}). This leads to the classical bias--variance trade-off from unconstrained learning. In constrained learning, however, we find a three-way trade-off that also involves the requirement difficulty. Scarce data therefore motivate not only the use of lower complexity parametrizations, but also constraint relaxations, both of which would lead to solutions that generalize better.
\end{enumerate}

The proof of Theorem~\ref{T:main} is carried out in three steps mapped in Figure~\ref{F:proof}. First, we analyze the \emph{duality gap} of the functional problem~\eqref{P:csl_variational}~(Section~\ref{S:duality_gap}, Proposition~\ref{T:zdg}), showing that strong duality holds even when the~$\ell_i$ are non-convex. The motivation for using this functional problem lies in the observation that if the parametrization is rich enough, i.e., if~$\calH$ is similar to~$\calHb$, then we expect the solution of~\eqref{P:csl} to be close to that of~\eqref{P:csl_variational}. Our second step quantifies this statement by bounding the error due to the use of a non-convex hypothesis class~(the \emph{approximation gap}, Section~\ref{S:parameterization_gap}, Proposition~\ref{T:param}). Finally, we study the effect of approximating expectations by sample averages~(the \emph{empirical gap}, Section~\ref{S:empirical_gap}, Proposition~\ref{T:empirical}). Since some of these results may be of independent interest, we briefly discuss each of them in the sequel.

\subsection{The duality gap}
\label{S:duality_gap}

When the~$\ell_i$ are convex, it is well-known that the value of the dual problem~\eqref{P:dual_csl_variational} attains the value of the primal~\eqref{P:csl_variational}, i.e., $\tilde{D}^\star = \tilde{P}^\star$, under some constraint qualification~(e.g., Assumption~\ref{A:slater})~\cite{Bertsekas09c}. The next result shows that under mild condition on the distributions~$\fkD_i$, this equality holds even if the~$\ell_i$ are non-convex. We note that, besides being the first step in the construction of Theorem~\ref{T:main}, this result has also been used in other contexts~(see, e.g., \cite{Chamon20f, Eisen19l, Peifer20s}).

\begin{proposition}\label{T:zdg}
Suppose there exists~$\phi^\prime \in \calHb$ such that~$\E_{(\bx,y) \sim \fkD_i} \!\big[ \ell_i\big( \phi^\prime(\bx),y \big) \big] < c_i$ for~$i = 1,\dots,m$. Under conditions~(a)--(b) of Theorem~\ref{T:main}, \eqref{P:csl_variational} is strongly dual, i.e., $\tilde{P}^\star = \tilde{D}^\star$.
\end{proposition}

\begin{proof}
See Appendix~\ref{X:zdg}.
\end{proof}

Hence, even if~\eqref{P:csl_variational} is a non-convex program~(e.g., in the rate-constrained learning example of Section~\ref{S:sims}), it remains strongly dual under mild conditions. In particular, if the conditional~$\bx \vert y$ induced by the distributions~$\fkD_i$ are non-atomic, i.e., do not contain Dirac deltas~(see, e.g., the distributions in Example~\ref{R:ecrm}), and~$\calHb$ is \emph{decomposable}~(e.g., some convex subset of~$L_2$ or~$L_\infty$).

Condition~(b) from Theorem~\ref{T:main} requires~$\calY$ be finite, i.e., this result holds for classification problems. We address the regression case, i.e., continuous output~$y$, in Appendix~\ref{X:zdg_regression} using a stronger uniform continuity assumption on the losses. The regression case can then be approximated arbitrarily well by a sequence of ever finer classification problems yielding the required strong duality result~(Proposition~\ref{T:zdg_regression}). A similar approach is used in the construction of regression trees~\cite{Friedman01t}.

\subsection{The approximation gap}
\label{S:parameterization_gap}

Whereas Proposition~\ref{T:zdg} shows there is no duality gap between~\eqref{P:csl_variational} and~\eqref{P:dual_csl_variational}, we are interested in the parametrized problem~\eqref{P:csl} rather than these infinite dimensional ones. The next step towards the empirical dual~\eqref{P:empirical_dual} is therefore to determine the error incurred from using~$\calH$ instead~$\calHb$, i.e., $D^\star-\tilde{D}^\star$.

Notice that~\eqref{P:param_dual_csl} is both the dual problem of~\eqref{P:csl} and a parametrized version of~\eqref{P:dual_csl_variational}, so that the approximation gap between~$D^\star-\tilde{D}^\star$ also informs the duality gap between~$P^\star-D^\star$~(Figure~\ref{F:proof}). As long as the parametrization is rich enough, we should expect both to be small. This intuition is formalized in the following proposition.

\begin{proposition}\label{T:param}
Let~$\bmu^\star$ be a solution of~\eqref{P:param_dual_csl}. Under the conditions of Theorem~\ref{T:main}, there exists a feasible~$\btheta^\dagger \in \argmin_{\btheta \in \Theta}\ L(\btheta,\bmu^\star)$ and the value~$D^\star$ of~\eqref{P:param_dual_csl} obeys
\begin{equation}\label{E:param_gap}
	P^\star - (1 + \norm{\bmut}_1)M \nu \leq D^\star \leq P^\star
		\text{,}
\end{equation}
for~$P^\star$ as in~\eqref{P:csl} and~$\bmut \in \setR^m_+$ maximizing~$\tilde{d}_{\nu}(\bmu) = \tilde{d}(\bmu) +  M\nu \norm{\bmu}_1$ for $\tilde{d}(\bmu)$ as in~\eqref{P:dual_csl_variational}.
\end{proposition}

\begin{proof}
See Appendix~\ref{X:param}.
\end{proof}

Despite being finite dimensional, \eqref{P:param_dual_csl} remains a statistical problem. Hence, though Proposition~\ref{T:param} establishes that its solutions are~\eqref{P:csl}-feasible and near-optimal, they remain uncomputable without explicit knowledge of the distributions~$\fkD_i$. Observe, however, that the objective function~$d$ of~\eqref{P:param_dual_csl} involves an unconstrained statistical problem. We have therefore done most of the heavy lifting and can now rely on the uniform bounds in Assumption~\ref{A:empirical}~(or Proposition~\ref{T:emp_bounds}).

\subsection{The empirical gap}
\label{S:empirical_gap}

The final step to transform~\eqref{P:param_dual_csl} into the empirical dual problem~\eqref{P:empirical_dual} is to turn the statistical Lagrangian~\eqref{E:lagrangian_csl_variational} into the empirical~\eqref{E:empirical_lagrangian}. The estimation error incurred in this step is detailed in the next proposition.

\begin{proposition}\label{T:empirical}
Let~$\bhmu$ be a solution of~\eqref{P:empirical_dual} with finite~$\hat{D}^\star$. Under the conditions of Theorem~\ref{T:main}, there exists~$\bhtheta \in \argmin_{\btheta \in \Theta}\ \hat{L}(\btheta,\bhmu)$ that is probably approximately feasible and near-optimal for~\eqref{P:csl}. Explicitly, it holds with probability~$1-(3m+2)\delta$ over the samples drawn from the distributions~$\fkD_i$ that
\begin{align}
	\big\vert D^\star - \hat{D}^\star \big\vert &\leq (1 + \bar{\Delta}) \overline{\zeta}
		\quad \text{and}
		\label{E:empirical_gap}
	\\
	\E_{\fkD_i} \!\Big[ \ell_i\big( f_{\bm{{\hat{\theta}}}^\star}(\bx), y \big) \Big]
		&\leq c_i + \zeta_i(N_i)
		\text{.}
		\label{E:empirical_feas}
\end{align}
where~$\overline{\zeta} = \max_i \zeta_i(N_i)$ and~$\bar{\Delta} = \max (\norm{\bmu^\star}_1, \norm{\bhmu}_1)$ for~$\bmu^\star$ and~$\bhmu$ solutions of~\eqref{P:param_dual_csl} and~\eqref{P:empirical_dual} that achieve~$D^\star$ and~$\hat{D}^\star$, respectively.
\end{proposition}

\begin{proof}
See Appendix~\ref{X:empirical}.
\end{proof}

Theorem~\ref{T:main} is obtained directly from Propositions~\ref{T:param}--\ref{T:empirical} using the triangle inequality. Observe that the order in which these transformations are applied to~\eqref{P:csl_variational} is crucial~(Figure~\ref{F:proof}). If we were to begin by replacing the expectations in~\eqref{P:csl_variational} with sample averages, we would obtain a functional version of~\eqref{P:ecrm}. However, generalization guarantee would then require~$\calHb$ to be PAC learnable, which is considerably stricter than for~$\calH$. In particular, while~$\calH$ may have finite VC dimension, $\calHb$ generally does not. We could try to overcome this issue by parametrizing~\eqref{P:csl_variational} first, but that would simply lead us back to~\eqref{P:csl} for which strong duality does not typically hold since it is a non-convex optimization problem.

One concern that may arise is that the upper bound in~\eqref{E:empirical_gap} depends on the Lagrange multipliers~$\bmu^\star$ and~$\bhmu$, whose values are not known \emph{a priori}. In particular, the value of~$\bhmu$ could depend on~$N_i$ in such a way that~\eqref{E:empirical_gap} does not vanish as the number of samples grows. In that case, the empirical dual problem would not be a near-PACC learner~(Definition~\ref{D:pacc}). This is, however, not the case. Indeed, the existence of strictly feasible solutions~(Assumption~\ref{A:slater}) implies an upper bound on the size of the Lagrange multipliers. We collect this classic result from the optimization literature in the following lemma.

\begin{lemma}\label{T:norm_mu_bound}
	Let~$\bmu^\star$, $\bhmu$, and~$\bmut$ be optimal solutions of~\eqref{P:param_dual_csl}, \eqref{P:empirical_dual}, and~$\max_{\bmu \in \setR^m_+} \tilde{d}(\bmu) + M\nu\norm{\bmu}_1$ respectively. Then, $\Delta = \max (\norm{\bmu^\star}_1, \norm{\bhmu}_1, \norm{\bmut}_1) \leq B/\xi$, for~$\xi$ as in Assumption~\ref{A:slater}.
\end{lemma}

\begin{proof}
We prove the bound for the empirical case since the same argument follows for~$\bmu^\star$ and~$\bmut$. By definition of the empirical dual function~\eqref{E:empirical_dual_function}, it holds that
\begin{align*}
	\hat{D}^\star = \hat{d}(\bhmu) &\leq \frac{1}{N_0} \sum_{n_0 = 1}^{N_0} \ell_0\big( f_{\btheta}(\bx_{n_0}), y_{n_0} \big)
	\\
	{}&+ \sum_{i = 1}^m \hat{\mu}^\star_i \left[ \frac{1}{N_i} \sum_{n_i = 1}^{N_i}
		\ell_i\big( f_{\btheta}(\bx_{n_i}), y_{n_i} \big) - c_i \right]
		\text{,}
\end{align*}
for all~$\btheta \in \Theta$. Using the strictly feasible point~$\bm{{\hat{\theta}^\prime}}$ from Assumption~\ref{A:slater} and the fact that~$\bhmu \in \setR^m_+$, we further obtain
\begin{equation}\label{E:lagrange_norm_bound}
	\hat{D}^\star \leq \frac{1}{N_0} \sum_{n_0 = 1}^{N_0} \ell_0\big( f_{\bm{{\hat{\theta}^\prime}}}(\bx_{n_0}), y_{n_0} \big)
		- \norm{\bhmu}_1 \xi
		\text{.}
\end{equation}
To conclude, notice that
\begin{equation*}
	\hat{D}^\star \geq \hat{d}(\bzero) = \min_{\btheta \in \Theta}\ \frac{1}{N_0}
		\sum_{n_0 = 1}^{N_0} \ell_0\big( f_{\btheta}(\bx_{n_0}), y_{n_0} \big)
		\text{.}
\end{equation*}
and use the fact that~$\ell_0$ is $[0,B]$-valued.
\end{proof}


\section{A Constrained Learning Algorithm}
\label{S:algorithm}

\begin{algorithm}[t]
\caption{Primal-dual constrained learning}
	\label{L:primal_dual}
\begin{algorithmic}[1]
	\State \textit{Inputs}: number of iterations~$T \in \setN$, step size~$\eta > 0$, and $N_i$ samples~$(\bx_{n_i},y_{n_i}) \sim \fkD_i$, for~$i = 0,\dots,m$.
	\State \textit{Initialize}: $\bmu^{(0)} = \zeros$
	\For $t = 1,\dots,T$
		\State Obtain~$\btheta^{(t-1)}$ such that
		\begin{equation*}
			\hat{L}\Big( \btheta^{(t-1)}, \bmu^{(t-1)}\Big)
				\leq \min_{\btheta \in \setR^p} \hat{L} \Big(\btheta, \bmu^{(t-1)} \Big) + \rho
		\end{equation*}

		\State Evaluate constraint slacks
		\begin{equation*}
			s_i^{(t-1)} = \frac{1}{N_i} \sum_{n_i = 1}^{N_i}
				\ell_i\big( f_{\btheta^{(t-1)}}(\bx_{n_i}), y_{n_i} \big) - c_i
		\end{equation*}

		\State Update dual variables
		\begin{align*}
			\mu_i^{(t)} &= \Big[ \mu_i^{(t-1)} + \eta s_i^{(t-1)} \Big]_+
		\end{align*}
	\EndFor
\end{algorithmic}
\end{algorithm}

We have argued that~\eqref{P:empirical_dual} is preferable to~\eqref{P:ecrm} for learning under requirements because it is unconstrained. That is not to say that~\eqref{P:empirical_dual} is easy to solve. But it is certainly not harder than classical ERM. In this section, we show that this is the case by describing a practical algorithm to~(approximately) solve~\eqref{P:empirical_dual} that only requires~(approximately) solving unconstrained learning problems.

Start by noticing that the outer maximization is a convex optimization program. Indeed, the empirical dual function defined in~\eqref{E:empirical_dual_function} is the pointwise minimum of a set of affine functions and is therefore concave~\cite{Bertsekas09c}. Additionally, its (sub)gradients can be easily computed by evaluating the constraint slacks at the minimizer of empirical Lagrangian~$\hat{L}$~\cite[Ch.~3]{Bertsekas15c}. The main challenge in~\eqref{P:empirical_dual} is therefore solving the inner minimization in~\eqref{E:empirical_dual_function}.

Note, however, that this minimization is a classical, unconstrained ERM problem. In fact, it is equivalent to solving an instance of a regularized learning problem. Hence, despite the (possible)~non-convexity of the Lagrangian~\eqref{E:empirical_lagrangian}, local minimizers can be found using, e.g., gradient descent, when the losses and parametrizations are differentiable~(i.e., most common machine learning models). In fact, there is ample empirical and theoretical evidence that stochastic gradient descent can find good local minimizers for deep learning models such as (convolutional)~NNs~\cite{Zhang17u, Ge18l, Brutzkus17g, Soltanolkotabi18t}. This is in contrast to~\eqref{P:ecrm} for which even obtaining a feasible~$\btheta$ can be intricate.

Algorithm~\ref{L:primal_dual} takes advantage of this fact to approximate the solution of the constrained learning problem~\eqref{P:csl} by alternating between minimizing the empirical Lagrangian~$\hat{L}$ from~\eqref{E:empirical_lagrangian} with respect to~$\btheta$ for fixed dual variables~$\bmu$ and updating the dual variable using the resulting minimizer. Observe that step~3 requires that we obtain a $\rho$-approximate minimizer of the empirical Lagrangian. The following theorem shows that, if this is possible, then Algorithm~\ref{L:primal_dual} yields a near-optimal solution of~\eqref{P:csl}.

\begin{theorem}\label{T:convergence_quant}
Let~$\calM^\star = \argmax_{\bmu \in \setR^m_+}\ \hat{d}(\bmu)$ be the set of Lagrange multipliers of~\eqref{P:empirical_dual} and define~$U_0 = \inf_{\bmu \in \calM^\star}\ \norm{\bmu}^2$.
Under the conditions of Theorem~\ref{T:main}, $U_0$ is finite and the primal-dual pair~$\big( \btheta^{(T-1)},\bmu^{(T-1)} \big)$ obtained after running Algorithm~\ref{L:primal_dual} for
\begin{equation*}\label{E:number_of_iterations}
	T = \ceil{\frac{U_0}{2 \eta M\nu}} + 1
		\text{ steps}
\end{equation*}
with step-size
\begin{equation}\label{E:eta}
	\eta \leq \frac{2 \bar{\zeta}}{m B^2}
\end{equation}
satisfies
\begin{equation}\label{E:convergence_quant}
	\abs{P^\star - \hat{L}\Big( \btheta^{(T-1)},\bmu^{(T-1)} \Big)}
		\leq \rho + (2 + \Delta ) (M \nu + \bar{\zeta})
\end{equation}
with probability~$1-(3m+2)\delta$ over sample sets, for~$\Delta$ and~$\bar{\zeta}$ defined as in Theorem~\ref{T:main}.
\end{theorem}

\begin{proof}
The proof is follow similarly to~\cite[Thm.~3]{Chamon20p}. Though~\cite{Chamon20p} assumes convexity of the~$\ell_i$, the proof itself relies only on strong duality, which holds here due to Proposition~\ref{T:zdg}. For completeness, a revised argument is provided in Appendix~\ref{X:convergence_quant}.
\end{proof}

Theorem~\ref{T:convergence_quant} bounds the error of Algorithm~\ref{L:primal_dual} in estimating~$P^\star$, the value of the learning problem~\eqref{P:csl}. In Theorem~\ref{T:convergence_quant}, the number of iterations~$T$ and the step size~$\eta$ are chosen so as to converge to a neighborhood of size~$O\big(\Delta (\bar{\zeta} + M\nu)\big)$, since this is the statistical error incurred by the dual learner~(Theorem~\ref{T:main}). Solving the empirical dual problem beyond that point would not improve the quality of the estimate. Naturally, while Theorem~\ref{T:convergence_quant} can be used to guide the choice of these parameters, their values are typically determined in practice by trial-and-error and cross-validation. Indeed, \eqref{E:number_of_iterations} and~\eqref{E:eta} depend on parameters of the learning task that are often hard to estimate, such as the Lipschitz constant~$M$~(Assumption~\ref{A:losses}), the empirical errors~$\zeta_i$~(Assumption~\ref{A:empirical}), and the approximation quality~$\nu$~(Assumption~\ref{A:parametrization}).

It is worth noting that~\eqref{E:convergence_quant} is a guarantee on the deterministic primal-dual pair~$\big( \btheta^{(T-1)},\bmu^{(T-1)} \big)$ as opposed to the randomized guarantees typically provided, e.g., in fair learning~\cite{Agarwal18a, Kearns18p, Cotter19o}. Still, Theorem~\ref{T:convergence_quant} only provides guarantees on approximating the value~$P^\star$, which is typically not the goal in learning. It does \emph{not} state that~$\btheta^{(T-1)}$ is near-optimal or even approximately feasible. This issue, known as \emph{primal recovery}, is not specific to constrained learning and is a fundamental limitation of duality in general~\cite{Nedic09a, Bertsekas09c}. While the experiments in Section~\ref{S:sims} suggest this is not a major issue in for typical learning problems, we next provide optimality and feasibility guarantees when randomizing over the iterates Algorithm~\ref{L:primal_dual}.

\begin{theorem}\label{T:convergence_rand}
Let~$\fke_t$ denote the empirical distribution over~$\big\{ \btheta^{(\tau)} \big\}$ generated by Algorithm~\ref{L:primal_dual} for~$\tau = 0,\dots,t-1$, i.e., $\fke_t$ is obtained by sampling from~$\{\btheta^{(0)}, \dots, \btheta^{(\tau)}\}$ uniformly at random. Under Assumptions~\ref{A:empirical} and~\ref{A:slater}, it holds with probability at least~$1-3(m+1)\delta$ that
\begin{subequations}\label{E:convergence_rand}
\begin{equation}
	\E_{\fkD_i,\, \btheta \sim \fke_{T}} \!\Big[ \ell_i\big( f_{\btheta}(\bx), y \big) \Big]
		\leq c_i + \zeta_i(N_i) + \frac{2 C}{\eta T}
		\text{,}
		\label{E:rand_feasibility}
\end{equation}
for all~$i = 1,\dots,m$. If all the conditions of Theorem~\ref{T:main} are met and the step size is chosen as in~\eqref{E:eta}, then we simultaneously have
\begin{equation}
	\E_{\fkD_0,\, \btheta \sim \fke_{T}} \!\Big[ \ell_0\big( f_{\btheta}(\bx), y \big) \Big]
		\leq P^\star  + \rho + (2+\Delta) (M\nu + \bar{\zeta})
		\text{.}
		\label{E:rand_optimality}
\end{equation}
\end{subequations}
\end{theorem}

\begin{proof}
See Appendix~\ref{X:convergence_rand}.
\end{proof}

Whereas Theorem~\ref{T:convergence_quant} only dealt with the value of~\eqref{P:param_dual_csl}, Theorem~\ref{T:convergence_rand} provides guarantees simultaneously on the value and feasibility of a randomized solution obtained by sampling the iterates~$\btheta^{(t)}$ uniformly at random. It is worth contrasting this result with those obtained in the context of rate-constrained learning. In particular, \cite{Kearns18p, Agarwal18a} obtain randomized solutions by directly optimizing a distribution over~$\Theta$. Doing so lifts~\eqref{P:csl} to a linear program for which strong duality holds~\cite{Bertsekas09c}. We showed in Theorem~\ref{T:main}~(more precisely, Proposition~\ref{T:zdg}) that this is, in fact, not necessary in the context of constrained learning~(see Section~\ref{S:rate} for more details on the rate-constrained case).

In fact, this is not necessary in general when considering randomized solutions. This is the approach taken by~\cite{Cotter19o}. However, \cite{Cotter19o} does not consider the issue of generalization, tackling~\eqref{P:ecrm} directly. Here, however, we are interested in solving~\eqref{P:csl}, i.e., obtain a solution that generalizes to the population in the PACC sense~(Definition~\ref{D:pacc}). Thus, we need only solve the dual to within the statistical error described in Theorem~\ref{T:main}, which allows us to use a fixed step size and obtain faster rates.

A common point between these previous works and Algorithm~\ref{L:primal_dual} is the use of an~(approximate) oracle~(step~3). It is often the case that an iterative procedure, such as gradient descent, underlies this oracle and the cost of running this procedure until convergence~(or even until a good solution is obtained) can be prohibitive. A common alternative is to adopt an Arrow-Hurwicz-style approach in which the primal variable~$\btheta^{(t)}$ and the dual variables~$\bmu^{(t)}$ are updated iteratively~\cite{Arrow58s}. While the convergence guarantee of Theorems~\ref{T:convergence_quant}--\ref{T:convergence_rand} no longer holds in this case, good results are observed in practice by performing the primal and dual updates at different timescales, e.g., by performing step~3 once per epoch. We showcase these results next.


\section{Applications}
\label{S:sims}

This section illustrates how constrained learning can be used to formulate and tackle two learning problems: robust and fair learning. The first example~(Section~\ref{S:robust}) showcases how~\eqref{P:csl} can be used to address the nominal accuracy vs.\ adversarial robustness trade-off. While the losses used in this example are convex~(cross-entropy loss), the parametrization is nonlinear~(convolutional NN, CNN), rendering~\eqref{P:csl} non-convex. The second example~(Section~\ref{S:rate}) tackles a rate constrained problem from fairness. In this case, we use a logistic model so that the objective is convex, but the constraints involve a non-convex indicator function~(0/1 loss). The discontinuous nature of this function poses additional issues that we address in under a margin assumption using a smooth surrogate~(sigmoidal function).

\subsection{Constrained learning with convex losses: Robust learning}
\label{S:robust}

Robustness is a well-known issue affecting modern machine learning models, especially CNNs. It is in fact straightforward to construct small input perturbations that drastically change the model output. Indeed, even perturbations as small as~$1\%$ of the pixel range can drop the accuracy of a trained model from above~$85\%$ to below~$15\%$~(see ``Classical training'' in Figure~\ref{F:cifar_adversarial}). To this end, numerous approaches have been proposed based on robust optimization~\cite{Sinha18c} and statistical smoothing~\cite{Salman19p, Cohen19c}. Yet, a growing body of empirical evidence has shown \emph{adversarial training} to be the most effective way to obtain robust classifiers, essentially by training models on perturbed data rather than directly using the sample set~\cite{Szegedy14i, Goodfellow15e, Madry18t, Sinha18c, Shaham18u}. Namely,
\begin{prob}\label{P:adversarial}
	\minimize_{\phi\in\calH}\ \E_{(\bxt,y) \sim \fkA} \!\Big[ \ell_0\big( \phi(\bxt),y \big) \Big]
		\text{,}
\end{prob}
for some adversarial distribution~$\fkA$ induced by taking~$\bxt = \bx + \bdelta$ for some~$\norm{\bdelta}_{\infty} \leq \varepsilon$. While this approach is now ubiquitous, it often results in classifiers with poor nominal performance~\cite{Tsipras19r, Javanmard20p}~(Figure~\ref{F:cifar_adversarial}).

In practice, penalty-based methods combining both clean and perturbed data, i.e., the objectives of~\eqref{P:sl} and~\eqref{P:adversarial}, into a single loss function are often used to overcome this issue~\cite{Zheng16i, Wang19i, Zhang19t}. However, while empirically successful, these methods cannot guarantee nominal or adversarial performance outside of the training samples. As we have mentioned before, classical learning theory~\cite{Vapnik00t, Shalev-Shwartz04u, Mohri18f} provides generalization bounds only for the aggregated objective and not each individual penalty term. Additionally, the choice of the penalty parameter is not straightforward and depends on the underlying learning task, making it difficult to transfer across instances and highly dependent on domain expert knowledge. To complicate things further, it may even be necessary for this parameter to evolve during training, as we illustrate next.

Similar to~\cite{Chamon20p}, we can use constrained learning to tackle this problem by using~\eqref{P:csl} to write
\begin{prob}\label{P:robustness}
	\minimize_{\btheta \in \Theta}&
		&&\E_{(\bx,y) \sim \fkD} \!\Big[ \ell_0\big( f_{\btheta}(\bx),y \big) \Big]
	\\
	\subjectto& &&\E_{(\bxt,y) \sim \fkA} \!\Big[ \ell_0\big( f_{\btheta}(\bxt),y \big) \Big] \leq c
		\text{.}
\end{prob}
In words, \eqref{P:robustness} seeks a model with the best possible performance on nominal data~(distribution~$\fkD$) among those models that have good performance under corrupted data~(distribution~$\fkA$). When the distribution~$\fkA$ is fixed \emph{a priori}, \eqref{P:robustness} formulates a problem of out-of-distribution generalization. For adversarial learning, $\fkA$ depends on the model~$f_{\btheta}$ and can be quite intricate to determine or even sample from~\cite{Carlini19o}. While constrained learning can also be used to tackle this issue, this is beyond the scope of this work~(see~\cite{Robey21a}). Still, though we may not be able to sample from the worst-case~$\fkA$, we can sample from distributions induced by different attacks developed in the literature, such as FGSM~\cite{Goodfellow15e} or \emph{PGD}~\cite{Madry18t}, and use the theory and algorithm developed in this work to obtain robustness guarantees against these distributions~$\fkA$, regardless of whether they are adversarial.

\begin{figure}[t]
	\centering
	\vspace*{3pt}
	\includesvg{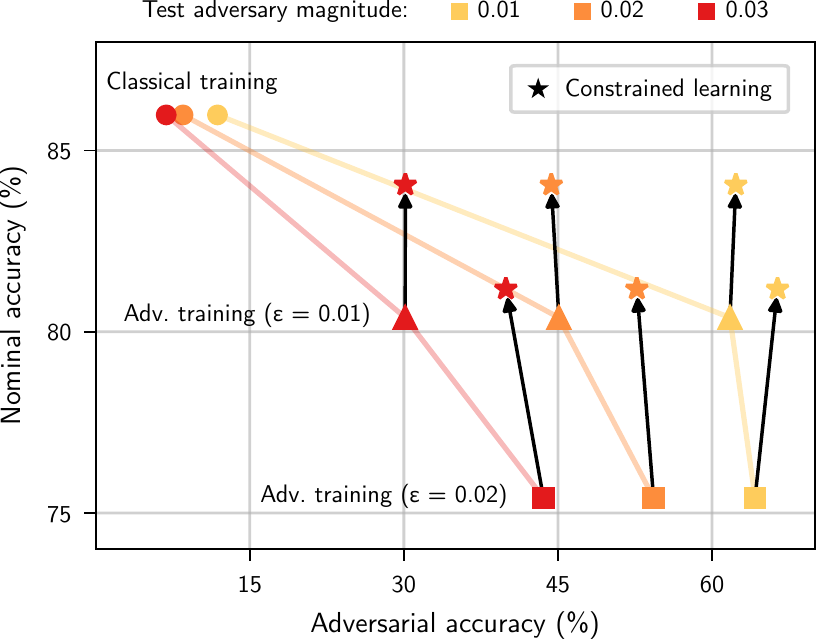}
	\caption{Trade-off between nominal and adversarial accuracy for models trained using classical~[\eqref{P:sl}, circles], adversarial~[\eqref{P:adversarial}, triangles and squares], and constrained~[\eqref{P:robustness}, stars] learning. Here, ``nominal accuracy'' stands for~$\E_{\fkD} \big[ \ell_0( f_{\btheta}(\bx),y ) \big]$, where~$\fkD$ is the natural, unperturbed distribution of the images, and ``adversarial accuracy'' stands for~$\E_{\fkA} \big[ \ell_0( f_{\btheta}(\bx),y ) \big]$, where~$\fkA$ incorporates worst-case perturbations of the input data for the classifier approximated using \emph{PGD}~\cite{Madry18t}.}
	\label{F:cifar_adversarial}
\end{figure}

We begin by training a ResNet18~\cite{He16d} to classify images from the CIFAR-10 dataset using ADAM with the same settings as~\cite{Kingma17a} and batch size~$128$. We reserved~$100$ randomly sampled images from each class for validation. The unconstrained classifier trained over~$100$ epochs reached it best accuracy over the validation set after~$82$ epochs, which corresponds to a nominal test accuracy of~$85.4\%$~(Figure~\ref{F:cifar_adversarial}). However, when the input is attacked using \emph{PGD}~\cite{Madry18t}, the accuracy falls below~$15\%$ already for~$\varepsilon = 0.01$. In all tests below, we apply \emph{PGD} for~$50$ iterations with a step size of~$\varepsilon/30$ and display the worst result over~$10$ restarts. Although adversarial training is able to achieve adversarial accuracy up to four times better~(after~$300$ training epochs), it does so at the cost of deteriorating the nominal performance~(Figure~\ref{F:cifar_adversarial}). Next, we use~\eqref{P:robustness} to illustrate that better trade-offs are possible.

To do so, we use Algorithm~\ref{L:primal_dual} to solve~\eqref{P:robustness} sampling from~$\fkA$ using \emph{PGD} with~$\varepsilon = \{0.01, 0.02\}$, ADAM with step sizes~$\{10^{-2},10^{-3}\}$~(all other settings as in~\cite{Kingma17a}) for step~3, and updating the dual variables~(step~5) once per epoch, using ADAM with a step sizes of~$\{10^{-3},10^{-4}\}$~(all other settings as in~\cite{Kingma17a}) and~$c = \{0.1,0.4\}$. These constraint values~$c$ were chosen by trial-and-error to achieve specific values of adversarial accuracy. Different values would lead to different compromises between nominal and adversarial accuracy that may be more appropriate for different applications. To accelerate training, we use a much weaker attack running \emph{PGD} without restarts for only~$5$ steps with step size~$\varepsilon/3$. The result is a considerably more robust classifier that has better nominal performance than the classifier obtained using adversarial training, i.e.,~\eqref{P:adversarial}. Due to the interactive dynamics of the primal-dual, however, training these classifiers can require~$2$ to~$5$ times the number of epochs needed to perform adversarial training, depending on how hard the constraints are to satisfy~(Figure~\ref{F:cifar_trace}). This gap can be reduced by tuning the parameters of the algorithm and/or using faster optimization methods. Such improvements are left for future work.

Note that, despite the similarities between Algorithm~\ref{L:primal_dual}~(step~3) and penalty-based methods, a key distinction is the fact that the dual variable~$\mu$ is adaptive as opposed to a fixed parameter. This distinction is at the core of the generalization results in Section~\ref{S:empirical_dual} and leads to learning dynamics with more flexibility to explore the optimization landscape. Figure~\ref{F:cifar_trace} illustrates this observation by displaying the evolution of the dual variable~$\mu^{(t)}$.

Figure~\ref{F:cifar_trace}b shows that, while the value of~$\mu$ is small at the end of training~(between~$0$ and~$0.5$), i.e., that the adversary has almost no influence on the training objective by the end of the learning process, its value rises above~$200$ in the first phase of training in order to meet the robustness constraint~(Figure~\ref{F:cifar_trace}a). It has in fact been observed empirically that restricting adversarial training to the early stages of learning can lead to more robust models~\cite{Zhang20a, Sitawarin20i}. Here, however, this behavior is not heuristic and arises naturally from solving the constrained problem using Algorithm~\ref{L:primal_dual}.

Finally, the fact that the dual variable grows only to then approach zero suggests that the adversarial constraint is used to initially guide the model to a favorable region of the optimization landscape, where little to no input from the constraint is required. Figure~\ref{F:cifar_continued} showcases this effect by using the constrained solution as a warm start and training using only the nominal loss, i.e., \eqref{P:sl}~(using ADAM with the same settings as~\cite{Kingma17a}). The resulting model~(\emph{Warm start}), while less robust than the constrained solution, is stronger than a model trained from a random initialization~(\emph{Random initialization}). This shows that, in this particular example, there exist minima of the nominal loss that are more robust to adversarial attacks, though they may not be easy to access by local search from a random initialization.

\begin{figure}[t]
	\centering
	\vspace*{3pt}
	\includesvg[width=0.95\columnwidth]{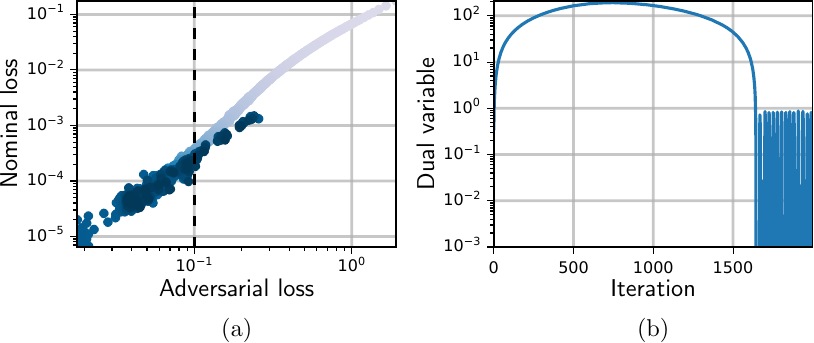}
	\caption{Evolution of losses and dual variable when using Algorithm~\ref{L:primal_dual} to solve~\eqref{P:robustness} for~$\varepsilon = 0.01$ and $c = 0.1$: (a)~nominal vs.\ adversarial loss and~(b)~dual variable~$\mu^{(t)}$.}
	\label{F:cifar_trace}
\end{figure}

\begin{figure}[t]
	\centering
	\includesvg{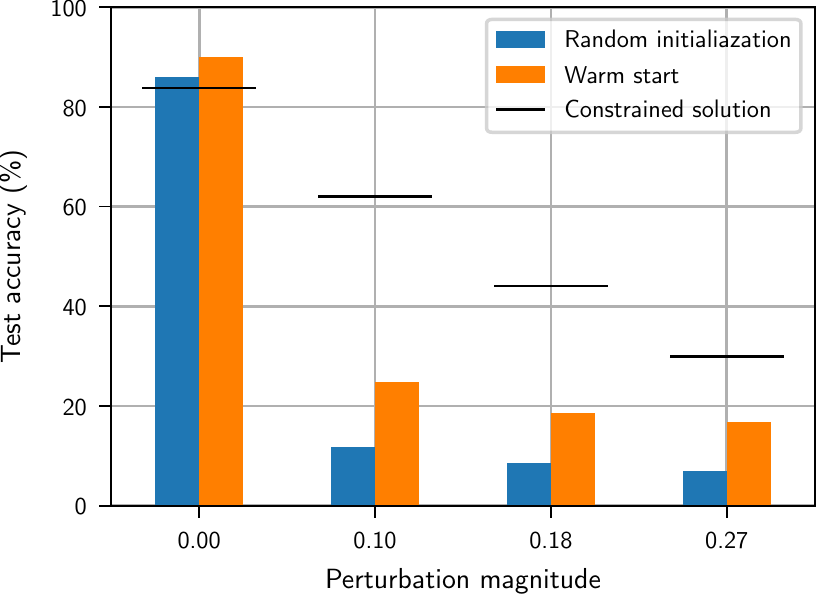}
	\caption{Nominal and adversarial accuracy of models trained from different initializations.}
	\label{F:cifar_continued}
\end{figure}

\subsection{Rate-constrained learning: Fair classification}
\label{S:rate}

Rate constraints, or more precisely probability or chance constraints, have been used in statistics at least since Neyman-Pearson~\cite{Neyman33i}. In learning, they have garnered attention due their central role in fairness, although they have also been used to control classifier performance, such as its coverage, precision, or accuracy~\cite{Goh16s, Kearns18p, Agarwal18a, Cotter19o}. Explicitly, a rate-constrained learning problems is written as
\begin{prob}[\textup{P-RCL}]\label{P:rcl}
	P_R^\star = \min_{\btheta \in \Theta}&
		&&\E_{\fkD_0} \!\Big[ \ell_0\big( f_{\btheta}(\bx),y \big) \Big]
	\\
	\subjectto& &&\E_{\fkD_i} \!\big[
		\indicator \big( g_i(f_{\btheta}(\bx),y) \geq 0 \big)
	\big] \leq c_i
		\text{.}
\end{prob}
Recall that~$\indicator(\calE)$ denotes the indicator of the event~$\calE$, i.e., $\indicator(\calE) = 1$ if~$\calE$ and zero otherwise. The constraint in~\eqref{P:rcl} is therefore equivalent to~$\sideset{}{_{\fkD_i}}\Pr \!\big[ g_i(f_{\btheta}(\bx),y) \geq 0 \big] \leq c_i$.

Rate constraints are challenging due to their non-convexity and non-differentiability. Hence, generalization guarantees are often obtained by modifying~\eqref{P:rcl} to optimize for a distributions over~$\Theta$~(e.g.,~\cite{Kearns18p, Agarwal18a}). The resulting problem is a linear program that lends itself to convex analysis tools. In contrast, the theory and algorithms from Sections~\ref{S:empirical_dual} and~\ref{S:algorithm} can be applied to~\eqref{P:rcl} directly, despite its non-convexity, as long as the hypothesis class~$\calH$ obeys the uniform convergence property from Assumption~\ref{A:empirical}. In the case of binary classification problems, i.e., when~$g_i(z, y) = zy$ or simply~$g_i(z, y) = z$, this is equivalent to having finite VC dimension or being PAC learnable~\cite[Thm.~6.7]{Shalev-Shwartz04u}.

The main obstacle to applying Algorithm~\ref{L:primal_dual} to~\eqref{P:rcl} is that the Lagrangian of~\eqref{P:rcl} is not differentiable. Indeed, the dual variables updates~(steps~4 and~5 in Algorithm~\ref{L:primal_dual}) use an approximate Lagrangian minimizer~(step~3), which can be hard to obtain without differentiability. Additionally, while Theorem~\ref{T:convergence_rand} ensures near-feasibility for Algorithm~\ref{L:primal_dual} even for non-smooth losses, our near-optimality results rely on smoothness. To overcome this issue, we obtain the approximate minimizer in step~3 by using a smooth surrogate of the indicator function, e.g., a sigmoid. Explicitly, consider the empirical dual problem of~\eqref{P:rcl}
\begin{prob}[\textup{D-RCL}]\label{P:rcl_dual}
	\hat{D}_R^\star = \max_{\bmu \in \setR_+^m}& &&\hat{d}_R(\bmu)
		\text{,}
\end{prob}
where~$\hat{d}_R(\bmu) = \min_{\btheta \in \Theta}\ L_R(\btheta,\bmu)$ for the empirical Lagrangian
\begin{equation}\label{E:lagrangian_rate}
\begin{aligned}
	\hat{L}_R(\btheta,\bmu) &= \frac{1}{N_0} \sum_{n_0 = 1}^{N_0} \ell_0\big( f_{\btheta}(\bx_{n_0}), y_{n_0} \big)
	\\
	{}&+ \sum_{i = 1}^m \mu_i \left[ \frac{1}{N_i} \sum_{n_i = 1}^{N_i}
				\indicator \big( g_i(f_{\btheta}(\bx_{n_i}), y_{n_i} ) \big) - c_i \right]
		\text{.}
\end{aligned}
\end{equation}
To overcome the discontinuous nature of the indicator in~\eqref{E:lagrangian_rate}, we replace it in step~3 of Algorithm~\ref{L:primal_dual} by
\begin{equation}\label{E:lagrangian_sigmoid}
\begin{aligned}
	\hat{L}_\sigma(\btheta,\bmu) &= \frac{1}{N_0} \sum_{n_0 = 1}^{N_0} \ell_0\big( f_{\btheta}(\bx_{n_0}), y_{n_0} \big)
		\\
		{}&+ \sum_{i = 1}^m \mu_i \left[ \frac{1}{N_i} \sum_{n_i = 1}^{N_i}
			\sigma \big( g_i(f_{\btheta}(\bx_{n_i}), y_{n_i} ) \big) - c_i \right]
		\text{,}
\end{aligned}
\end{equation}
for some surrogate~$\sigma$, such as the sigmoid
\begin{equation}\label{E:sigmoid}
	\sigma(x) = \left[ 1 + e^{-ax} \right]^{-1}
		\text{,} \quad a \geq 1
		\text{.}
\end{equation}

This is a typical approach in the statistics~(e.g., logistic models) and learning literature~\cite{Goh16s, Jang17c, Maddison17t, Cotter19o}. In particular, it was applied to rate-constrained learning in~\cite{Cotter19o}. In contrast to Algorithm~\ref{L:primal_dual}, their algorithm is based on Fritz-John conditions~\cite{Mangasarian67t} and a more complicated no-swap-regret dual update~(replacing steps~4 and~5). They obtain feasibility results similar to those in~\eqref{E:rand_feasibility}~(Theorem~\ref{T:convergence_rand}), but prove near-optimality only with respect to the value of the surrogate Lagrangian~\eqref{E:lagrangian_sigmoid}. Under a margin assumption, it is possible to derive a guarantee directly with respect to the value of~\eqref{P:rcl}.

\begin{assumption}\label{A:margin}
For all~$n_i = 1,\dots,N_i$, $i = 1,\dots,m$, and~$\bmu \in \setR_+^m$, it holds that
\begin{equation}\label{E:margin}
	\max \left\{ \big\vert g_i\big(f_{\btheta_R^\dagger(\bmu)}(\bx_{n_i}),y_{n_i}\big) \big\vert,
		\big\vert g_i\big(f_{\btheta_\sigma^\dagger(\bmu)}(\bx_{n_i}),y_{n_i}\big) \big\vert
	\right\} \geq \tau
		\text{.}
\end{equation}
for the Lagrangian minimizers~$\btheta_R^\dagger(\bmu) \in \argmin_{\btheta \in \Theta} \hat{L}_R(\btheta,\bmu)$ and~$\btheta_\sigma^\dagger(\bmu) \in \argmin_{\btheta \in \Theta} \hat{L}_\sigma(\btheta,\bmu)$.
\end{assumption}

The following proposition shows that minimizing the surrogate Lagrangian~\eqref{E:lagrangian_sigmoid} yields an approximate minimizer of the Lagrangian~\eqref{E:lagrangian_rate}.

\begin{proposition}\label{T:indicator_approx}
Under Assumption~\ref{A:margin} it holds for all~$\bmu \in \setR^m_+$ and~$\btheta_\sigma^\dagger(\bmu) \in \argmin_{\btheta \in \Theta} \hat{L}_\sigma(\btheta,\bmu)$ that
\begin{multline}\label{E:indicator_approx}
	\hat{L}_R\big( \btheta_\sigma^\dagger(\bmu),\bmu \big) \leq \min_{\btheta \in \Theta} \hat{L}_R(\btheta,\bmu)
		+ 2 \norm{\bmu}_1 \big( 1-\sigma(\tau) \big)
		\text{.}
\end{multline}
\end{proposition}

\begin{proof}
Let~$\btheta_R^\dagger(\bmu) \in \argmin_{\btheta \in \Theta} \hat{L}_R(\btheta,\bmu)$. Then, by optimality, notice that~$\hat{L}_\sigma(\btheta_R^\dagger,\bmu) - \hat{L}_\sigma(\btheta_\sigma^\dagger,\bmu) \geq 0$ to get
\begin{align*}
	\hat{L}_R(\btheta_\sigma^\dagger,\bmu) - \min_{\btheta} \hat{L}_R(\btheta,\bmu)
		&\leq \hat{L}_R(\btheta_\sigma^\dagger,\bmu) - \hat{L}_\sigma(\btheta_\sigma^\dagger,\bmu)
	\\
	{}&+ \hat{L}_{\sigma}(\btheta_R^\dagger,\bmu) - \hat{L}_R(\btheta_R^\dagger,\bmu)
		\text{.}
\end{align*}
Using the definition of the Lagrangians in~\eqref{E:lagrangian_rate} and~\eqref{E:lagrangian_sigmoid} then yields
\begin{multline*}
	L_R(\btheta_\sigma^\dagger,\bmu) - \min_{\btheta} \hat{L}_R(\btheta,\bmu) \leq{}
	\\
	\sum_{i=1}^m \frac{\mu_i}{N_i} \sum_{n_i = 1}^{N_i} \!\Big[
	\sigma \big( g_i(f_{\btheta_R^\dagger}(\bx_n),y_n) \big)
		- \indicator \big( g_i(f_{\btheta_R^\dagger}(\bx_n),y_n) \big)
	\\
	+ \indicator \big( g_i(f_{\btheta_\sigma^\dagger}(\bx_n),y_n) \big)
		- \sigma \big( g_i(f_{\btheta_\sigma^\dagger}(\bx_n),y_n) \big)
		\Big]
		\text{.}
\end{multline*}
Using~\eqref{E:margin} yields the desired bound.
\end{proof}

\begin{figure}[t]
	\centering
	\includesvg{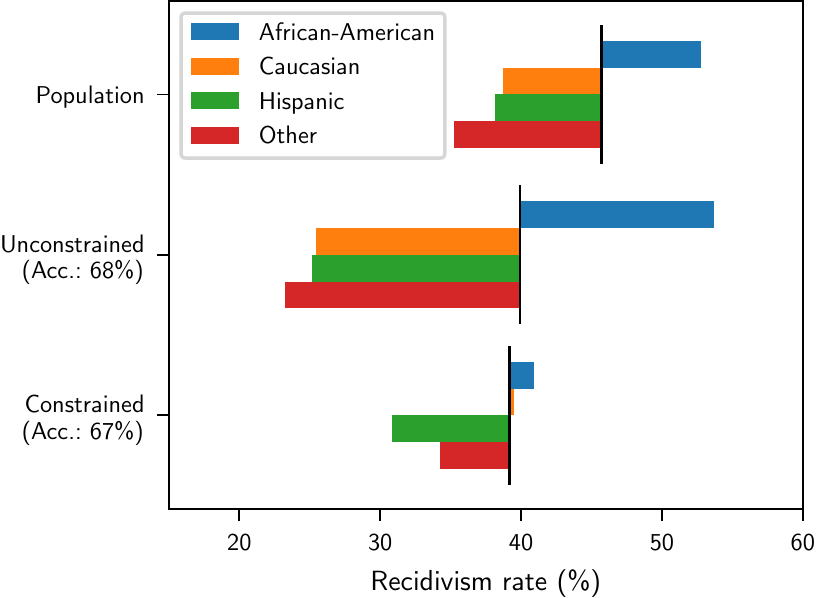}
	\caption{Dual variable relative to the fairness constraint.}
	\label{F:compas_prevalence}
\end{figure}

To illustrate the use of rate constraints in learning, we consider a fair classification application around the COMPAS dataset~\cite{ProPublica16c}~(data preprocessing details can be found in~\cite{Chamon20p}). The goal is to predict recidivism based on a person's characteristics and past offenses. Yet, while the overall recidivism rate in the dataset is~$45.5\%$, this rate is~$52.3\%$ for African-Americans, which compose more than half the sample set~(Figure~\ref{F:compas_prevalence}). Considering how this data was collected~(based on arrests), we may expect this disparity to be due to sampling bias.

Unconstrained, a logistic classifier exacerbates this skewness. While its test accuracy is~$68.5\%$, it predicts an overall recidivism rate of~$38.4\%$~(the actual rate on the test set is~$45.7\%$) while maintaining the African-American group rate at~$52.2\%$~(Figure~\ref{F:compas_prevalence}). This classifier was trained over a random sample containing~$80\%$ of the dataset using ADAM~\cite{Kingma17a} for~$1000$ epochs with batch size of~$128$ samples, learning rate~$0.2$, and all other parameters as in the original paper. While we use a logistic classifier, the same results are obtained for a single layer feed-forward neural network.

To overcome this issue, we impose fairness requirements during learning. More specifically, we use an asymmetric form of statistical parity that only upper bounds the difference between each protected group and the overall recidivism rate. Explicitly,
\begin{prob}\label{P:compas}
	\minimize_{\btheta \in \Theta}&
		&&\E \!\Big[ \ell_0\big( f_{\btheta}(\bx),y \big) \Big]
	\\
	\subjectto& &&\E \!\Big[
		\indicator \big[ f_{\btheta}(\bx) \geq 0.5 \big] \,\Big\vert\, \text{Race} = r
	\Big]
	\\
	&&&\qquad \leq \E \!\Big[ \indicator \big[ f_{\btheta}(\bx) \geq 0.5 \big] \Big] + 0.01
		\text{,}
\end{prob}
where~$\ell_0$ is the negative log-likelihood of the logistic distribution and~$r = \{\text{African-American}, \text{Caucasian}, \text{Hispanic}, \text{Other}\}$. In other words, the final classifier is required to predict recidivism within each group at most~$1\%$ above the rate at which it predicts recidivism in the overall population.

\begin{figure}[t]
\centering
{\footnotesize
\begin{tabular}{r @{\hspace{1.5em}}c @{\hspace{0.7em}}c @{\hspace{0.5em}}c @{\hspace{0.3em}}c%
	@{\hspace{0.6em}}c @{\hspace{0.3em}}c}
	& & & \multicolumn{4}{c}{\hspace{-0.7em}\emph{Prediction}} \\[0.2em]
	& & & 0 & 1 & 0 & 1\\
	\multirow{4}{*}{African-American}
	& \multirow{10}{*}{\rotatebox{90}{\hspace{-3em}\emph{True label}}}
	& 0 &
		\fcolorbox{black!20}{white}{\color{black!20}\parbox[c][7mm][c]{7mm}{\centering 31\%}} &
		\fcolorbox{black}{Set1.orange!20}{\parbox[c][7mm][c]{7mm}{\centering 16\%}} &
		\fcolorbox{black!20}{white}{\color{black!20}\parbox[c][7mm][c]{7mm}{\centering 36\%}} &
		\fcolorbox{black}{Set1.red!20}{\parbox[c][7mm][c]{7mm}{\centering 11\%}}
		\\[1.4em]
	& & 1 &
		\fcolorbox{black}{Set1.orange!20}{\parbox[c][7mm][c]{7mm}{\centering 16\%}} &
		\fcolorbox{black!20}{white}{\color{black!20}\parbox[c][7mm][c]{7mm}{\centering 37\%}} &
		\fcolorbox{black}{Set1.red!20}{\parbox[c][7mm][c]{7mm}{\centering 23\%}} &
		\fcolorbox{black!20}{white}{\color{black!20}\parbox[c][7mm][c]{7mm}{\centering 30\%}}
	\\[2em]
	\multirow{4}{*}{Caucasian}
	& & 0 &
		\fcolorbox{black!20}{white}{\color{black!20}\parbox[c][7mm][c]{7mm}{\centering 52\%}} &
		\fcolorbox{black}{Set1.red!20}{\parbox[c][7mm][c]{7mm}{\centering 9\%}} &
		\fcolorbox{black!20}{white}{\color{black!20}\parbox[c][7mm][c]{7mm}{\centering 44\%}} &
		\fcolorbox{black}{Set1.orange!20}{\parbox[c][7mm][c]{7mm}{\centering 17\%}}
		\\[1.4em]
	& & 1 &
		\fcolorbox{black}{Set1.red!20}{\parbox[c][7mm][c]{7mm}{\centering 23\%}} &
		\fcolorbox{black!20}{white}{\color{black!20}\parbox[c][7mm][c]{7mm}{\centering 16\%}} &
		\fcolorbox{black}{Set1.orange!20}{\parbox[c][7mm][c]{7mm}{\centering 16\%}} &
		\fcolorbox{black!20}{white}{\color{black!20}\parbox[c][7mm][c]{7mm}{\centering 23\%}}
	\\[2em]
	& & & \multicolumn{2}{c}{\hspace{-0.5em}Unconstrained}
	& \multicolumn{2}{c}{\hspace{-0.5em}Constrained}
\end{tabular}}

	\caption{Confusion matrix for unconstrained and fairness-constrained classifiers.}
		\label{F:compas_confusion}
\end{figure}

We solve~\eqref{P:compas} using a logistic classifier for~$f_{\btheta}$ trained with Algorithm~\ref{L:primal_dual} over~$T = 1000$ epochs. For step~3, we used ADAM with the same hyperparameters as above and a sigmoidal approximation for the indicator function. Explicitly, we replaced~$\indicator \big[ f_{\btheta}(\bx)\geq 0.5 \big]$ by~$\sigma\big[8(f_{\btheta}(\bx) - 0.5)\big]$, where~$\sigma$ denotes the sigmoid function. After each epoch, we updated the dual variables~(step~5) also using ADAM with step size~$0.001$. The results are shown in the last row of Figure~\ref{F:compas_prevalence}.

Notice that compared to the unconstrained model, the predicted recidivism rate over the test set remains almost the same~($38.9\%$), but the rates within each group are now more homogeneous. For instance, the rate for African-Americans is now only~$1.5\%$ above the cross-race average. In contrast, the model now predicts recidivism for Caucasians at a higher rate, from~$24.07\%$ in the unconstrained model to~$39.8\%$, closer to the actual rate in the data set~($39.1\%$). In fact, the main difference between the constrained and unconstrained models is their distribution of false negatives~(Figure~\ref{F:compas_confusion}). Indeed, while the unconstrained model \emph{implicitly} inflates the false negative rate for Caucasians, the constrained model \emph{explicitly} does so for African-Americans instead. Doing so balances the predicted recidivism rates while maintaining essentially the same overall accuracy~(Figure~\ref{F:compas_prevalence}).

Using a logistic classifier allows us to interpret its coefficients as odds ratio and analyze the difference in predictive behavior between the constrained and unconstrained models. The coefficients with largest changes are displayed in Figure~\ref{F:compas_coeffs}. Note that while the original model estimates that being African-American increases your chances of recidivism by almost~$30\%$, the constrained model compensates for the dataset biases by instead decreasing the probability by~$40\%$. The opposite effect occurs in the Caucasian group~(leading to the difference in false positives displayed in Figure~\ref{F:compas_confusion}). The model also compensates for the individual having a large number of priors, a group composed mostly of African-Americans in the sample~($69\%$).

\begin{figure}[t]
	\centering
	\includesvg{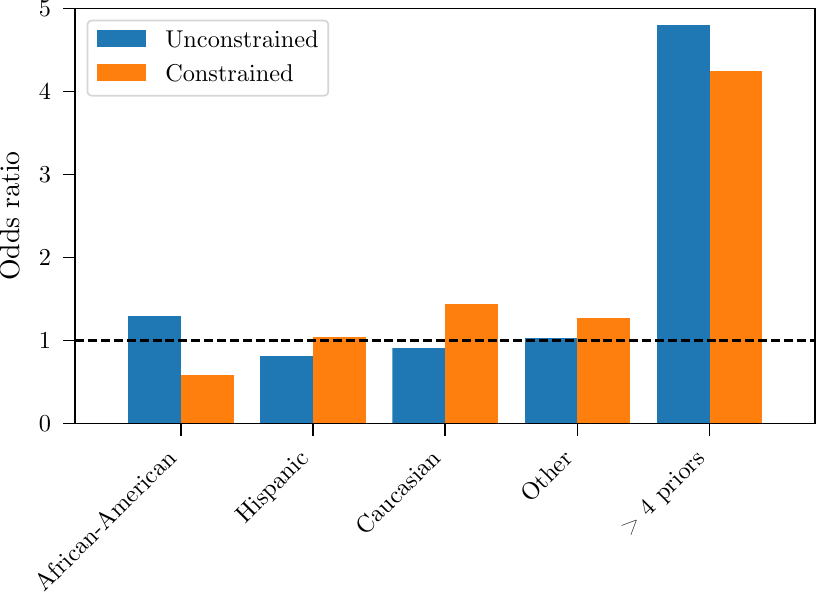}
	\caption{Dual variable relative to the fairness constraint.}
	\label{F:compas_coeffs}
\end{figure}

\section{Conclusion}

This work developed a constrained counterpart of classical learning theory and showed that statistical requirements can be explicitly imposed on learning problems by means of empirical dual learning. In doing so, it reduced the original constrained, statistical problem~(the one we want to solve) into an unconstrained, deterministic problem~(one we can solve). In contrast to penalty-based approaches, it proves that this method yields near-optimal, near-feasible solutions by bounding the duality, parametrization, and empirical errors under mild conditions. Based on this result, it then proposed a primal-dual algorithm to tackle the empirical dual problem. Robust and fair learning applications are used to showcase the usefulness of these developments. We believe this work provides a principled, practical framework for tackling learning under requirements, crucial for critical applications in the social, industrial, and medical fields. Some of its theoretical results may also be useful in the analysis of other forms of learning~(see~\cite{Paternain19c} for examples in reinforcement learning).

\bibliographystyle{IEEEtran}
\bibliography{IEEEabrv,af,bayes,control,gsp,math,ml,rkhs,rl,sp,stat}

\begin{IEEEbiography}{Luiz F.\ O.\ Chamon} received the B.Sc. and M.Sc. degrees in electrical engineering from the University of São Paulo, São Paulo, Brazil, in 2011 and 2015 and the Ph.D. degree in electrical and systems engineering from the University of Pennsylvania~(Penn), Philadelphia, in 2020. He is currently a postdoc at the Simons Institute of the University of California, Berkeley. In 2009, he was an undergraduate exchange student of the Masters in Acoustics of the École Centrale de Lyon, Lyon, France, and worked as an Assistant Instructor and Consultant on nondestructive testing at INSACAST Formation Continue. From 2010 to 2014, he worked as a Signal Processing and Statistics Consultant on a research project with EMBRAER. In 2018, he was recognized by the IEEE Signal Processing Society for his distinguished work for the editorial board of the IEEE Transactions on Signal Processing. He also received both the best student paper and the best paper awards at IEEE ICASSP 2020. His research interests include optimization, signal processing, machine learning, statistics, and control.
\end{IEEEbiography}

\begin{IEEEbiography}{Santiago Paternain} received the B.Sc. degree in electrical engineering from Universidad de la República Oriental del Uruguay, Montevideo, Uruguay in 2012, the M.Sc. in Statistics from the Wharton School in 2018 and the Ph.D. in Electrical and Systems Engineering from the Department of Electrical and Systems Engineering, the University of Pennsylvania in 2018. He is currently an Assistant Professor in the Department of Electrical Computer and Systems Engineering at the Rensselaer Polytechnic Institute. Prior to joining Rensselaer, Dr. Paternain was a postdoctoral Researcher at the University of Pennsylvania. His research interests lie at the intersection of machine learning and control of dynamical systems. Dr. Paternain was the recipient of the 2017 CDC Best Student Paper Award and the 2019 Joseph and Rosaline Wolfe Best Doctoral Dissertation Award from the Electrical and Systems Engineering Department at the University of Pennsylvania.
\end{IEEEbiography}

\begin{IEEEbiography}{Miguel Calvo-Fullana} received his B.Sc. degree in electrical engineering from the Universitat de les Illes Balears (UIB), in 2012 and the M.Sc. and Ph.D. degrees in electrical engineering from the Universitat Polit\`ecnica de Catalunya (UPC), in 2013 and 2017, respectively. From September 2012 to July 2013 he was a research assistant with Nokia Siemens Networks (NSN) and Aalborg University (AAU). From December 2013 to July 2017, he was with the Centre Tecnol\`ogic de Telecomunicacions de Catalunya (CTTC) as a research assistant. From September 2017 to September 2020, he was a postdoctoral researcher at the University of Pennsylvania. Since September 2020, he is a postdoctoral researcher at the Massachusetts Institute of Technology. His research interests lie in the broad areas of learning and optimization for autonomous systems. In particular, he is interested in multi-robot systems with an emphasis on wireless communication and network connectivity.
 \end{IEEEbiography}

\begin{IEEEbiography}{Alejandro Ribeiro} received the B.Sc. degree in electrical engineering from the Universidad de la República Oriental del Uruguay in 1998 and the M.Sc. and Ph.D. degrees in electrical engineering from the Department of Electrical and Computer Engineering at the University of Minnesota in 2005 and 2007. He joined the University of Pennsylvania (Penn) in 2008 where he is currently Professor of Electrical and Systems Engineering. His research is in wireless autonomous networks, machine learning on network data and distributed collaborative learning. Papers coauthored by Dr. Ribeiro received the 2021 Cambridge Ring Publication of the Year Award, the 2020 IEEE Signal Processing Society Young Author Best Paper Award, the 2014 O. Hugo Schuck best paper award, and paper awards at EUSIPCO 2021, ICASSP 2020, EUSIPCO 2019, CDC 2017, SSP Workshop 2016, SAM Workshop 2016, Asilomar SSC Conference 2015, ACC 2013, ICASSP 2006, and ICASSP 2005. His teaching has been recognized with the 2017 Lindback award for distinguished teaching and the 2012 S. Reid Warren, Jr. Award presented by Penn’s undergraduate student body for outstanding teaching. Dr. Ribeiro received an Outstanding Researcher Award from Intel University Research Programs in 2019.  He is a Penn Fellow class of 2015 and a Fulbright scholar class of 2003.
\end{IEEEbiography}

\appendices


\section{Proof of Proposition~\ref{T:zdg}: The Duality Gap}
\label{X:zdg}

When the function~$\ell_i$, $i = 0,\dots,m$, are convex~[(a)], \eqref{P:csl_variational} is a convex optimization problem. Under Assumption~\ref{A:slater}, known in this context as Slater's condition, its strong duality is a classical result from convex optimization theory~\cite[Prop.~5.3.1]{Bertsekas09c}. The following therefore focuses on the proof for the case in which these functions are not convex~[(b)].

Start by recalling that the dual problem~\eqref{P:dual_csl_variational} is a relaxation of its primal~\eqref{P:csl_variational} and therefore provides a lower bound on its optimal value. Explicitly, $\tilde{D}^\star \leq \tilde{P}^\star$~\cite[Chap.~5]{Bertsekas09c}. Hence, it suffices to prove that~$\tilde{D}^\star \geq \tilde{P}^\star$. We do so by showing that even though~\eqref{P:csl_variational} is a non-convex program, the range of its cost and constraints forms a convex set under the hypotheses of the proposition. Explicitly, define the cost-constraints epigraph as
\begin{multline}\label{E:cost_constraint}
	\calC = \Big\{ (s_0,\bs) \in \setR^{m+1} \,\Big\vert\,  \exists\, \phi \in \calHb
	\text{ such that }
	\\
	\E \left[\ell_0(\phi(\bx),y)\right] \leq s_0
		\text{ and }
		\E \left[\ell_i(\phi(\bx),y)\right] \leq s_i
	\Big\}
		\text{,}
\end{multline}
where the vector~$\bs \in \setR^m$ collects the~$s_i$, $i = 1,\dots,m$. For conciseness, we omit the distributions over which the expectations are taken whenever they can be inferred from the context. Then, the following holds:

\begin{lemma}\label{T:convexC}
If the conditional random variables~$\bx \vert y$ induced by the distributions~$\fkD_i$ are non-atomic and $\calY$ is finite, then the cost-constraints set~$\calC$ in~\eqref{E:cost_constraint} is a non-empty convex set.
\end{lemma}

Before proving Lemma~\ref{T:convexC}, let us show how it implies strong duality for~\eqref{P:csl_variational} by leveraging the following result from convex geometry:

\begin{proposition}[{Supporting hyperplane theorem~\cite[Prop.~1.5.1]{Bertsekas09c}}]\label{T:hyperplane}
Let~$\calA \subset \setR^n$ be a nonempty convex set. If~$\bxt \in \setR^n$ is not in the interior of~$\calA$, then there exists a hyperplane passing through~$\bxt$ such that~$\calA$ is in one of its closed halfspaces, i.e., there exists~$\bp \neq \bzero$ such that~$\bp^\top \bxt \leq \bp^\top \bx$ for all~$\bx \in \calA$.

\end{proposition}

To proceed, observe from~\eqref{P:csl_variational} that~$(\tilde{P}^\star,\bc)$, where~$\bc \in \setR^m$ collects the values of the constraint requirements~$c_i$, cannot be in the interior of~$\calC$, otherwise there would exist~$\epsilon > 0$ such that~$(\tilde{P}^\star-\epsilon,\bc) \in \calC$, violating the optimality of~$\tilde{P}^\star$. Proposition~\ref{T:hyperplane} then implies that there exists a non-zero vector~$(\mu_0,\bmu) \in \setR^{m+1}$ such that
\begin{equation}\label{E:hyperplane_bound}
	\mu_0 s_0 + \bmu^\top \bs \geq \mu_0 \tilde{P}^\star + \bmu^\top \bc
		\text{,} \quad \text{for all } (s_0,\bs) \in \calC
		\text{.}
\end{equation}
Observe that the hyperplanes in~\eqref{E:hyperplane_bound} are defined using the same notation as the dual problem~\eqref{P:dual_csl_variational} to foreshadow the fact that they actually span the values of the Lagrangian~\eqref{E:lagrangian_csl_variational}.

To proceed, note from~\eqref{E:cost_constraint} that~$\calC$ is unbounded above, i.e., if~$(s_0,\bs) \in \calC$ then~$(s_0^\prime,\bs^\prime) \in \calC$ for all~$(s_0^\prime,\bs^\prime) \succeq (s_0,\bs)$. Hence, \eqref{E:hyperplane_bound} can only hold if~$\mu_i \geq 0$, $i = 0,\dots,m$. Otherwise, there exists a vector in~$\calC$ such that the left-hand side of~\eqref{E:hyperplane_bound} evaluates to an arbitrarily negative number, eventually violating Proposition~\ref{T:hyperplane}. Let us now show that furthermore~$\mu_0 \neq 0$.

Indeed, suppose~$\mu_0 = 0$. Then~\eqref{E:hyperplane_bound} reduces to
\begin{equation}\label{E:mu0_inequality}
	\bmu^\top \bs \geq \bmu^\top \bc
	\Leftrightarrow
	\bmu^\top (\bs - \bc) \geq 0
		\text{,} \quad \text{for all } (s_0,\bs) \in \calC
		\text{.}
\end{equation}
Recall from Proposition~\ref{T:hyperplane} that there exists at least one~$\bmu \neq \zeros$ for which this inequality must hold. However, this is contradicted by the existence of the strictly feasible point~$\phi^\prime$. Explicitly, for every~$\bmu \neq \zeros$, there exists~$(s_0^\prime,\bs^\prime) \in \calC$, achieved by~$\phi^\prime$ from the hypotheses of the proposition, such that~$s_i^\prime < c_i$ for all~$i$, contradicting~\eqref{E:mu0_inequality}.

However, if~$\mu_0 \neq 0$, \eqref{E:hyperplane_bound} can be written as
\begin{equation*}
	s_0 + \bm{{\bar{\mu}}}^\top \bs \geq \tilde{P}^\star + \bm{{\bar{\mu}}}^\top \bc
		\text{,} \quad \text{for all } (s_0,\bs) \in \calC
		\text{,}
\end{equation*}
where~$\bm{{\bar{\mu}}} = \bmu/\mu_0$, which from the definition of~$\calC$ in~\eqref{E:cost_constraint} implies that
\begin{multline}\label{E:case_ii}
	\E \left[\ell_0(\phi(\bx),y)\right] + \sum_{i = 1}^m \bar{\mu}_i
		\left( \E \left[ \ell_i(\phi(\bx),y) \right] - c_i \right) \geq \tilde{P}^\star
		\text{,}
\end{multline}
for all~$\phi \in \calHb$. Note, however, that the left-hand side of~\eqref{E:case_ii} is the Lagrangian~\eqref{E:lagrangian_csl_variational}, i.e., \eqref{E:case_ii} implies that~$L(\phi, \bm{{\bar{\mu}}}) \geq \tilde{P}^\star$. In particular, this hold for the minimum of~$L(\phi, \bm{{\bar{\mu}}})$, implying that~$\tilde{D}^\star \geq \tilde{P}^\star$ and therefore, that strong duality holds for~\eqref{P:csl}.
\hfill$\blacksquare$

\vspace{\baselineskip}

All that remains now is proving that the cost-constraint set~$\calC$ in~\eqref{E:cost_constraint} is convex.

\begin{proof}[Proof of Lemma~\ref{T:convexC}]

This proof follows along the lines of~\cite{Chamon20f}. Let~$(s_0,\bs),(s_0^\prime,\bs^\prime) \in \calC$ be arbitrary points satisfied by~$\phi,\phi^\prime \in \calHb$, i.e., for~$i = 0,\dots,m$,
\begin{equation}\label{E:phi_achieve}
	\E \left[\ell_i(\phi(\bx),y)\right] \leq s_i
		\quad \text{and} \quad
	\E \left[\ell_i(\phi^\prime(\bx),y)\right] \leq s^\prime_i
		\text{.}
\end{equation}
It suffices then to show that~$\lambda (s_0,\bs) + (1-\lambda) (s_0^\prime,\bs^\prime) \in \calC$ for all~$\lambda \in [0,1]$ to obtain that~$\calC$ is convex. Equivalently, we must obtain~$\phi_\lambda \in \calHb$ such that
\begin{equation}\label{E:convexC}
	\E \left[\ell_i(\phi_\lambda(\bx),y)\right] \leq \lambda s_i + (1-\lambda) s_i^\prime
		\text{,} \quad i = 0,\dots,m
		\text{,}
\end{equation}
for all~$0 \leq \lambda \leq 1$. To do so, we rely on the following classical theorem about the range of non-atomic vector measures:

\begin{theorem}[{Lyapunov's convexity theorem~\cite[Chap.~IX, Cor.~5]{Diestel77v}}]
\label{T:lyapunov}

Let~$\fkq: \calB \to \setR^n$ be a finite dimensional vector measure over the measurable space~$(\Omega,\calB)$. If~$\fkq$ is non-atomic, then its range is convex, i.e., the set~$\{ \fkq(\calZ) : \calZ \in \calB \}$ is a convex set.

\end{theorem}

To see how Theorem~\ref{T:lyapunov} allows us to construct the desired~$\phi_\lambda$, let~$\Omega = \calX$, a subset of~$\setR^d$, and~$\calB$ be its Borel $\sigma$-algebra. Define the~$2\abs{\calY}(m+1) \times 1$ vector measure~$\fkp$ such that for every set~$\calZ \in \calB$ we have
\begin{equation}\label{E:vector_meas}
	\renewcommand*{\arraystretch}{1.5}
	\fkp(\calZ) =
	\begin{bmatrix}
		\int_{\calZ} \ell_i(\phi(\bx),y) f_i(\bx \vert y) d\bx
		\\
		\int_{\calZ} \ell_i(\phi^\prime(\bx),y) f_i(\bx \vert y) d\bx
	\end{bmatrix}_{i = 0,\dots,m;\, y \in \calY}
		\text{,}
\end{equation}
where~$f_i(\bx \vert y)$ denotes the conditional density of~$\bx$ given~$y$ induced by the joint distributions~$\fkD_i$%
\footnote{We assume here that these density exist only to simplify the notation. The integrals in~\eqref{E:vector_meas} can be taken against the conditional~(Radon-Nikodym) measures as long as the law of~$\bx$ is absolutely continuous with respect to the law of~$y$.}.
Hence, each entry of~$\fkp$ is an integral of the losses~$\ell_i$ of~$\phi$ or~$\phi^\prime$ with respect to a value of~$y \in \calY$. Immediately, we note that~$\fkp(\emptyset) = \bzero$ and
\begin{equation}\label{E:vector_meas_omega}
	\renewcommand*{\arraystretch}{1.5}
	\fkp(\Omega) =
	\begin{bmatrix}
		\E \left[ \ell_i(\phi(\bx),y) \mid y \right]
		\\
		\E \left[ \ell_i(\phi^\prime(\bx),y) \mid y \right]
	\end{bmatrix}_{i = 0,\dots,m;\, y \in \calY}
		\text{.}
\end{equation}

Due to the additive property of the Lebesgue integral, $\fkp$ in~\eqref{E:vector_meas} is a proper vector measure. What is more, the~$\ell_i$ are bounded functions, so the fact that~$\bx \vert y$ is non-atomic implies that~$\fkp$ is also non-atomic. Hence, from Theorem~\ref{T:lyapunov}, there exists a set~$\calT_\lambda \in \calB$ such that
\begin{equation}\label{E:calT}
	\fkp(\calT_\lambda)
		= \lambda \fkp(\Omega) + (1-\lambda) \fkp(\emptyset)
		= \lambda \fkp(\Omega)
		\text{,}
\end{equation}
for~$\lambda \in [0,1]$. Since~$\calB$ is a $\sigma$-algebra, it holds that~$\Omega \setminus \calT_\lambda \in \calB$ and by additivity we obtain
\begin{equation}\label{E:calTc}
	\fkp(\Omega \setminus \calT_\lambda) =
		(1 - \lambda) \fkp(\Omega)
		\text{.}
\end{equation}
From~\eqref{E:calT} and~\eqref{E:calTc},  we then construct~$\phi_\lambda$ as
\begin{equation}\label{E:phi_lambda}
	\phi_\lambda(\bx) =
	\begin{cases}
		\phi(\bx) \text{,}
			&\text{for } \bx \in \calT_\lambda
		\\
		\phi^\prime(\bx) \text{,}
			&\text{for } \bx \in \Omega \setminus \calT_\lambda
	\end{cases}
\end{equation}
It is straightforward from the decomposability of~$\calHb$, that~$\phi_\lambda \in \calHb$. We claim that it also satisfies~\eqref{E:convexC}.

To see this is the case, use the construction in~\eqref{E:phi_lambda} to obtain
\begin{align*}
	\E \big[\ell_i(\phi_\lambda(\bx),y) \mid y \big] &=
	\int_{\calT_\lambda} \ell_i(\phi(\bx),y) f_i(\bx \vert y) d\bx
	\\
	{}&+ \int_{\Omega \setminus \calT_\lambda} \ell_i(\phi^\prime(\bx),y) f_i(\bx \vert y) d\bx
		\text{.}
\end{align*}
Note from~\eqref{E:vector_meas}, that these integral can be written as entries of the vector measure~$\fkp$, namely
\begin{equation}\label{E:machinery}
	\E \big[\ell_i(\phi_\lambda(\bx),y) \mid y \big]
		= \big[ \mathfrak{p}(\calT_\lambda) \big]_{\phi;\,i;\,y} +
			\big[ \mathfrak{p}(\Omega \setminus \calT_\lambda) \big]_{\phi^\prime;\,i;\,y}
		\text{.}
\end{equation}
In~\eqref{E:machinery}, we use~$\big[ \fkp \big]_{\phi;\,i;\,y}$ to denote the entry of~$\fkp$ relative to the function~$\phi$, the~$i$-th loss, and the label~$y \in \calY$. From~\eqref{E:calT} and~\eqref{E:calTc}, we know that~\eqref{E:machinery} evaluates to
\begin{align*}
	\E \left[\ell_i(\phi_\lambda(\bx),y) \mid y \right]
		&= \left[ \lambda \mathfrak{p}(\Omega) \right]_{\phi;\,i;\,y} +
		\left[ (1 - \lambda) \mathfrak{p}(\Omega) \right]_{\phi^\prime;\,i;\,y}
	\\
	{}&= \lambda \E \left[\ell_i(\phi(\bx),y) \mid y \right]
	\\
	{}&+ (1 - \lambda) \E \left[\ell_i(\phi^\prime(\bx),y) \mid y \right]
		\text{,}
\end{align*}
for all~$i = 0,\dots,m$ and~$y \in \calY$. Using the tower~(total expectation) property, we immediately conclude that for~$i = 0,\dots,m$,
\begin{align*}
	\E \left[\ell_i(\phi_\lambda(\bx),y) \right] &=
		\E_y \bigg[ \lambda \E \left[\ell_i(\phi(\bx),y) \mid y \right]
	\\
	{}&+ (1 - \lambda) \E \left[\ell_i(\phi^\prime(\bx),y) \mid y \right] \bigg]
	\\
	{}&= \lambda \E \left[\ell_i(\phi(\bx),y) \right]
		+ (1 - \lambda) \E \left[\ell_i(\phi^\prime(\bx),y) \right]
		\text{,}
\end{align*}
which from~\eqref{E:phi_achieve} yields
\begin{equation*}
	\E \left[\ell_i(\phi_\lambda(\bx),y) \right] \leq \lambda s_i + (1 - \lambda) s_i^\prime
		\text{.}
\end{equation*}

Hence, there exists~$\phi_\lambda \in \calHb$ such that~\eqref{E:convexC} holds for all~$\lambda \in [0,1]$ and~$(s_0,\bs),(s_0^\prime,\bs^\prime) \in \calC$. The set~$\calC$ is therefore convex. Moreover, the strictly feasible~$\phi^\prime$ from the hypotheses of the proposition implies that~$\calC$ is not be empty.
\end{proof}


\section{Duality Gap for Regression}
\label{X:zdg_regression}

The Lyapunov convexity theorem~(Theorem~\ref{T:lyapunov}) turns out to be quite sensitive to the hypothesis that the vector measure takes values in a finite dimensional Banach space~\cite[Ch.~IX]{Diestel77v}. Yet, for compact~$\calY$, we can overcome this issue without resorting to super-atomless~(saturated) spaces by assuming the losses are uniformly continuous in~$y$ and slicing~$\calY$ to approximate the regression problem by a sequence of increasingly finer classification problems. Explicitly, we use the following assumption:

\begin{assumption}\label{A:regression}
The functions~$y \mapsto \ell_i(\phi(\cdot),y) f_i(\cdot \vert y)$, $i = 1,\dots,m$, are uniformly continuous in the total variation topology for each~$\phi \in \calHb$, where~$f_i(\bx \vert y)$ denotes the density of the conditional random variable induced by~$\fkD_i$. Explicitly, for each~$\phi \in \calHb$ and every~$\epsilon > 0$ there exists~$\delta_{\phi,i} > 0$ such that for all~$\abs{y - \tilde{y}} \leq \delta_{\phi,i}$ it holds that
\begin{equation*}
	\sup_{\calZ \in \calB} \int_\calZ \big\vert \ell_i(\phi(\bx),y) f_i(\bx \mid y)
		- \ell_i(\phi(\bx),\tilde{y}) f_i(\bx \mid \tilde{y}) \big\vert d\bx
			\leq \epsilon
		\text{,}
\end{equation*}
\end{assumption}

Once again, we consider the measurable space~$(\Omega,\calB)$ where~$\Omega = \setR^d$ and~$\calB$ is a Borel $\sigma$-algebra.

\begin{proposition}\label{T:zdg_regression}
Consider the dual problem~\eqref{P:dual_csl_variational} and let~$\calY$ be bounded and the conditional distributions~$\bx \vert y$ induced by the~$\fkD_i$ are non-atomic. Under assumptions~\ref{A:losses}, \ref{A:slater}, and~\ref{A:regression}, \eqref{P:csl} is strongly dual, i.e., $\tilde{P}^\star = \tilde{D}^\star$.
\end{proposition}

The proof of Proposition~\ref{T:zdg_regression} follows that of the finite~$\calY$ case in Appendix~\ref{X:zdg} by replacing Lemma~\ref{T:convexC} by the following result.

\begin{lemma}\label{T:convexC_b}
Under the assumptions of Proposition~\ref{T:zdg_regression}, the cost-constraints set~$\calC$ in~\eqref{E:cost_constraint} is a non-empty convex set.
\end{lemma}

\begin{proof}[Proof of Lemma~\ref{T:convexC_b}]

Without loss of generality, assume~$\calY = [0,1]$. Once again, let~$(s_0,\bs),(s_0^\prime,\bs^\prime) \in \calC$ be achieved by~$\phi,\phi^\prime \in \calHb$. Our goal, as before, is to construct~$\phi_{\lambda}$ such that
\begin{equation}\label{E:convexC_regression}
	\E \left[\ell_i(\phi_\lambda(\bx),y)\right] \leq \lambda s_i + (1-\lambda) s_i^\prime
		\text{,} \quad i = 0,\dots,m
		\text{,}
\end{equation}
holds for all~$\lambda \in [0,1]$.

To do so, fix~$\epsilon > 0$ and let~$\delta > 0$ be such that
\begin{gather*}
	\sup_{\calZ \in \calB} \int_\calZ \big\vert \ell_i(\phi(\bx),y) f_i(\bx \vert y)
		- \ell_i(\phi(\bx),\tilde{y}) f_i(\bx \vert \tilde{y}) \big\vert d\bx \leq \frac{\epsilon}{3}
	\\
	\text{and}
	\\
	\sup_{\calZ \in \calB} \int_\calZ \big\vert \ell_i(\phi^\prime(\bx),y) f_i(\bx \vert y)
		- \ell_i(\phi(\bx),\tilde{y}) f_i(\bx \vert \tilde{y}) \big\vert d\bx \leq \frac{\epsilon}{3}
\end{gather*}
for all~$\abs{y - \tilde{y}} \leq \delta$ and all~$i = 1,\dots,m$. Assumption~\ref{A:regression} guarantees such a~$\delta$ exists since we can take it to be the minimum of the~$2m$ positive~$\delta_{\phi,i}$ and~$\delta_{\phi^\prime,i}$. Then, partition~$\calY$ into the intervals
\begin{equation}\label{E:ytilde_intervals}
	\calI_k = \big[ (k-1)\delta, k\delta \big]
		\text{.}
\end{equation}
with midpoint~$\tilde{y}_k \triangleq ( k - 1/2 )\delta$ and let~$\tilde{\calY} = \{ \tilde{y}_k \}$. Since~$\calY$ is bounded, $\abs{\tilde{\calY}} < \infty$.

To proceed, construct the~$2\abs{\tilde{\calY}}(m+1) \times 1$ vector measure
\begin{equation}\label{E:vector_meas_b}
	\renewcommand*{\arraystretch}{1.5}
	\fkp_\epsilon(\calZ) =
	\begin{bmatrix}
		\int_{\calZ} \ell_i(\phi(\bx),\tilde{y}) f_i(\bx \vert \tilde{y}) d\bx
		\\
		\int_{\calZ} \ell_i(\phi^\prime(\bx),\tilde{y}) f_i(\bx \vert \tilde{y}) d\bx
	\end{bmatrix}_{i = 0,\dots,m;\, \tilde{y} \in \tilde{\calY}}
\end{equation}
and, using the non-atomicity of~$\fkp_\epsilon$, obtain from Theorem~\ref{T:lyapunov} a set~$\calT_{\lambda,\epsilon} \in \calB$ such that
\begin{equation}\label{E:calT_b}
	\fkp(\calT_{\lambda,\epsilon}) = \lambda \fkp_\epsilon(\Omega)
	\quad \text{and} \quad
	\fkp(\Omega \setminus \calT_{\lambda,\epsilon}) =
		(1 - \lambda) \fkp_\epsilon(\Omega)
		\text{.}
\end{equation}
From~\eqref{E:calT_b}, construct~$\phi_{\lambda,\epsilon}$ as
\begin{equation}\label{E:phi_mu_b}
	\phi_{\lambda,\epsilon}(\bx) =
	\begin{cases}
		\phi(\bx) \text{,}
			&\text{for } \bx \in \calT_{\lambda,\epsilon}
		\\
		\phi^\prime(\bx) \text{,}
			&\text{for } \bx \in \Omega \setminus \calT_{\lambda,\epsilon}
	\end{cases}
\end{equation}
Since~$\calHb$ is decomposable, we again have~$\phi_{\lambda,\epsilon} \in \calHb$. Let us show that it satisfies~\eqref{E:convexC_regression} up to an additive error~$\epsilon$.

Indeed, notice from~\eqref{E:phi_mu_b} that
\begin{equation}\label{E:expected_phi_mu_b}
\begin{aligned}
	\E \!\big[\ell_i(\phi_{\lambda,\epsilon}(\bx),y) \big] &=
		\E_y \!\bigg[\E \big[\ell_i(\phi_{\lambda,\epsilon}(\bx),y) \mid y \big] \bigg]
	\\
	{}&= \E_y \!\bigg[ \int_{\calT_{\lambda,\epsilon}} \ell_i(\phi(\bx),y) f_i(\bx \vert y) d\bx \bigg]
	\\
	{}&+ \E_y \!\bigg[ \int_{\Omega \setminus \calT_{\lambda,\epsilon}}
		\ell_i(\phi^\prime(\bx),y) f_i(\bx \vert y) d\bx \bigg]
		\text{.}
\end{aligned}
\end{equation}
Focusing on the first expectation, start by building a simple function approximation of the integrand using the intervals~$\calI_k$ from~\eqref{E:ytilde_intervals}. From our choice of~$\delta$ and Assumption~\ref{A:regression}, it holds that
\begin{multline}\label{E:first_bound_b}
	\E_y \!\bigg[ \int_{\calT_{\lambda,\epsilon}} \ell_i(\phi(\bx),y) f_i(\bx \vert y) d\bx \bigg] \leq{}
	\\
	\sum_{k = 0}^{\abs{\tilde{\calY}}} \E_y \!\Big[ \indicator[ y \in \calI_k ] \Big]
		\left[
			\int_{\calT_{\lambda,\epsilon}} \ell_i(\phi(\bx),\tilde{y}_k) f_i(\bx \vert \tilde{y}_k) d\bx
				+ \frac{\epsilon}{3}
		\right]
			\text{.}
\end{multline}
Notice from the definition of the vector measure in~\eqref{E:vector_meas_b} that the value of the integrals in~\eqref{E:first_bound_b} are entries of~$\fkp_\epsilon(\calT_{\lambda,\epsilon})$. Using the property of~$\calT_{\lambda,\epsilon}$ in~\eqref{E:calT_b}, we then get
\begin{multline}\label{E:second_bound_b}
	\E_y \!\bigg[ \int_{\calT_{\lambda,\epsilon}} \ell_i(\phi(\bx),y) f_i(\bx \vert y) d\bx \bigg] \leq{}
	\\
	\sum_{k = 0}^{\abs{\tilde{\calY}}} \E_y \!\Big[ \indicator[ y \in \calI_k ] \Big]
		\left[
			\lambda \int \ell_i(\phi(\bx),\tilde{y}_k) f(\bx \vert \tilde{y}_k) d\bx
				+ \frac{\epsilon}{3}
		\right]
			\text{.}
\end{multline}
Using Assumption~\ref{A:regression} once again, together with the fact that~$\fkD_i$ is a probability measure, i.e., $\sum_{k = 0}^{\abs{\tilde{\calY}}} \E_y \big[ \indicator[ y \in \calI_k ] \big] = 1$, yields
\begin{multline}\label{E:third_bound_b}
	\E_y \!\bigg[ \int_{\calT_{\lambda,\epsilon}} \ell_i(\phi(\bx),y) f_i(\bx \vert y) d\bx \bigg]
	\\
	{}\leq \lambda \bigg[ \E_y \!\Big[ \E \!\big[\ell_i(\phi(\bx),y) \mid y \big]  \Big] + \frac{\epsilon}{3} \bigg]
		+ \frac{\epsilon}{3}
	\\
	{}= \lambda \E \!\big[\ell_i(\phi(\bx),y) \big] + \frac{(1+\lambda) \epsilon}{3}
			\text{.}
\end{multline}
A similar argument yields
\begin{multline}\label{E:fourth_bound_b}
	\E_y \!\bigg[ \int_{\Omega \setminus \calT_{\lambda,\epsilon}} \ell_i(\phi^\prime(\bx),y) f_i(\bx \vert y) d\bx \bigg] \leq{}
	\\
	\leq (1-\lambda) \E \!\big[\ell_i(\phi^\prime(\bx),y) \big] + \frac{(2-\lambda) \epsilon}{3}
			\text{.}
\end{multline}
Using~\eqref{E:third_bound_b} and~\eqref{E:fourth_bound_b} in~\eqref{E:expected_phi_mu_b}, we obtain that for all~$\epsilon > 0$ and~$\lambda \in [0,1]$, there exists~$\phi_{\lambda,\epsilon} \in \calHb$ such that
\begin{equation}\label{E:approx_convexity}
	\E \!\big[ \ell_i(\phi_{\lambda,\epsilon}(\bx),y) \big]
		\leq \lambda s_i + (1-\lambda) s_i^\prime + \epsilon
		\text{,}
\end{equation}
for all~$i = 0,\dots,m$.

Suppose now that there is no~$\phi_\lambda \in \calHb$ such that~$\E\left[\ell_i(\phi_{\lambda}(\bx),y) \right] \leq \lambda s_i + (1-\lambda) s_i^\prime$. Then, there exists~$\tau > 0$ such that
\begin{equation*}
	\E\left[\ell_i(\phi(\bx),y) \right] > \lambda s_i + (1-\lambda) s_i^\prime + \tau
\end{equation*}
for~$\phi \in \calHb$. For instance, if~$\E\left[\ell_i(\phi(\bx),y) \right] \geq \lambda s_i + (1-\lambda) s_i^\prime + \gamma$, let~$\tau = \gamma/2$. However, this violates~\eqref{E:approx_convexity} for~$\epsilon = \tau$ leading to a contradiction. Since~$\calHb$ is closed, we therefore obtain that for all~$(s_0,\bs),(s_0^\prime,\bs^\prime) \in \calC$ and~$\lambda \in [0,1]$, there exists~$\phi_\lambda \in \calHb$ such that~\eqref{E:convexC_regression}, showing that~$\calC$ is convex. The strictly feasible~$\phi^\prime$ from Assumption~\ref{A:slater} implies that~$\calC$ is also not empty.
\end{proof}


\section{Proof of Proposition~\ref{T:param}: The Approximation Gap}
\label{X:param}

We first prove that there exists a~$f_{\btheta^\dagger}$ feasible for~\eqref{P:csl} and then bound the gap between~$D^\star$ and~$P^\star$ using the functional problems~\eqref{P:csl_variational} and~\eqref{P:dual_csl_variational}.

\vspace{0.5\baselineskip}\noindent
\textbf{Feasibility.}
The proof relies on the following lemma characterizing the superdifferential of the dual function~\eqref{E:param_dual_function_csl}. Explicitly, we say~$\bp \in \setR^m$ is a \emph{supergradient} of~$d$ at~$\bmu$ if
\begin{equation}\label{E:supergrad}
	d(\bmu^\prime) \leq d(\bmu) + \bp^\top (\bmu^\prime - \bmu)
		\text{, for all } \bmu^\prime \in \setR^m_+
		\text{.}
\end{equation}
The set of all supergradients of~$d$ at~$\bmu$ is called the \emph{superdifferential} of~$d$ at~$\bmu \in \setR^m_+$ and is denoted~$\partial d(\bmu)$. Additionally, let
\begin{equation}\label{E:lagrangian_minimizers}
	\Theta^\dagger(\bmu) = \argmin_{\btheta \in \Theta} L(\btheta, \bmu)
\end{equation}
for the Lagrangian defined in~\eqref{E:param_dual_function_csl} and define the constraint slack vector~$\bs \in \setR^m$ with entries
\begin{equation}\label{E:slack_vector}
	s_i(\btheta) = \Big[ \E \!\big[ \ell_{i}\big( f_{\btheta}(\bx), y \big) \big] - c_{i} \Big]_+
		\text{,}
\end{equation}
where~$[z]_+ = \max(0,z)$ denotes the projection onto~$\setR_+$. For clarity, we now omit the distribution~$\fkD_{i}$ over which the expected value is taken. The following is sometimes known as Danskin's theorem.

\begin{lemma}\label{T:danskin}
Under Assumption~\ref{A:losses}, it holds that
\begin{equation}\label{E:danskin}
	\partial d(\bmu) = \conv\left(
		\bigcup_{\btheta^\dagger \in \Theta^\dagger(\bmu)} \bs\big( \btheta^\dagger \big)
	\right)
		\text{.}
\end{equation}
\end{lemma}

\begin{proof}
	See, e.g., \cite[Thm.~2.87]{Ruszczynski06n}, noting that~$L(\btheta, \cdot)$ is affine for all~$\btheta \in \Theta$, $L(\cdot,\bmu)$ is continuous for all~$\bmu \in \setR^m_+$~(Assumption~\ref{A:losses}), and~$\Theta$ is compact. Hence, we meet conditions~(i)--(iii) of the theorem. Considering that~$d(\bmu) = \min_{\btheta \in \Theta} L(\btheta,\bmu) = -\max_{\btheta \in \Theta} -L(\btheta,\bmu)$ yields the desired result.
\end{proof}

The proof then follows by contradiction. Indeed, suppose that all elements of~$\Theta^\dagger(\bmu^\star)$ are infeasible for~\eqref{P:csl}. Then, for all~$\btheta^\dagger \in \Theta^\dagger(\bmu^\star)$ there exists~$i(\btheta^\dagger)$ such that~$\E \!\Big[ \ell_{i(\btheta^\dagger)}\big( f_{\btheta^\dagger}(\bx), y \big) \Big] - c_{i(\btheta^\dagger)} > 0$. From Lemma~\ref{T:danskin}, $\bzero \notin \partial d(\bmu^\star)$. However, this contradicts the optimality of~$\bmu^\star$. Hence, there must be~$\btheta^\dagger \in \Theta^\dagger(\bmu^\star)$ feasible for~\eqref{P:csl}.

\vspace{0.5\baselineskip}\noindent
\textbf{Near-optimality.}
The upper bound is trivial from weak duality~\cite{Bertsekas09c}. For the lower bound, consider the functional problem
\begin{prob}\label{P:csl_perturbed}
	\tilde{P}_{\nu}^\star = \min_{\phi \in \calHb}&
		&&\E \!\Big[ \ell_0\big( \phi(\bx),y \big) \Big]
	\\
	\subjectto& &&\E \!\Big[ \ell_i\big( \phi(\bx),y \big) \Big]
		\leq c_i - M \nu
		\text{,}
	\\
	&&&\qquad i = 1,\ldots,m
\end{prob}
and let~$\tilde{\phi}_{\nu}^\star \in \calHb$ be a solution. Such a solution exists by Assumption~\ref{A:slater}. Notice that strong duality holds for~\eqref{P:csl_perturbed}~(by Proposition~\ref{T:zdg} and Assumption~\ref{A:slater}), so that
\begin{equation}\label{E:dual_perturbed}
	\tilde{P}_\nu^\star = \max_{\bmu \in \setR^m_+}\ \min_{\phi \in \calHb}\ \tilde{L}_\nu(\phi,\bmu)
		= \tilde{L}_\nu(\tilde{\phi}_\nu^\star,\bmut)
		\text{,}
\end{equation}
where~$\bmut$ achieves the maximum in~\eqref{E:dual_perturbed} for the Lagrangian
\begin{equation}\label{E:lagrangian_perturbed}
\begin{aligned}
	\tilde{L}_\nu(\phi,\bmu) &=
		\E \!\Big[ \ell_0\big( \phi(\bx),y \big) \Big]
	\\
	{}&+ \sum_{i = 1}^m \mu_i \Big[
		\E \!\big[ \ell_i\big( \phi(\bx),y \big) \big]
			- c_i + M \nu
	\Big]
		\text{.}
\end{aligned}
\end{equation}

To proceed, note from~\eqref{P:param_dual_csl} that
\begin{equation*}
	D^\star \geq \min_{\btheta \in \Theta} L( \btheta, \bmu )
		\text{,} \quad \text{for all } \bmu \in \setR^m_+
		\text{.}
\end{equation*}
Immediately, we obtain that
\begin{equation}\label{E:lower_bound1}
	D^\star	\geq \min_{\btheta \in \Theta} L( \btheta, \bmut )
		\geq \min_{\phi \in \calHb} \tilde{L}( \phi, \bmut )
		\text{,}
\end{equation}
where the second inequality comes from the fact that~$\calH \subseteq \calHb$. Then, note that~$\tilde{L}_\nu$ in~\eqref{E:lagrangian_perturbed} is related to~$\tilde{L}$ in~\eqref{E:lagrangian_csl_variational} by
\begin{equation*}
	\tilde{L}_\nu(\phi,\bmu) = \tilde{L}(\phi,\bmu) + M\nu \norm{\bmu}_1
		\text{,}
\end{equation*}
where we used the fact that~$\bmu \in \setR^m_+$ to write that~$\sum_i \mu_i = \norm{\bmu}_1$. From~\eqref{E:lower_bound1} we then get
\begin{equation*}
	D^\star	\geq \min_{\phi \in \calHb} \tilde{L}( \phi, \bmut )
		= \min_{\phi \in \calHb} \tilde{L}_\nu( \phi, \bmut ) - M\nu \norm{\bmut}_1
		\text{,}
\end{equation*}
which using the strong duality of~\eqref{P:csl_perturbed} yields
\begin{equation}\label{E:lower_bound2}
	D^\star	\geq \tilde{P}_\nu^\star
		= \E \!\Big[ \ell_0\big( \tilde{\phi}_\nu^\star(\bx),y \big) \Big] - M\nu \norm{\bmut}_1
		\text{.}
\end{equation}
To obtain the lower bound in~\eqref{E:param_gap}, suffices it to show that~$\E \!\big[ \ell_0\big( \tilde{\phi}_\nu^\star(\bx),y \big) \big] \geq P^\star - M\nu$.

Explicitly, it holds from Assumptions~\ref{A:losses} and~\ref{A:parametrization} that there exists~$\tilde{\btheta}^\star_\nu \in \Theta$ such that
\begin{multline}\label{E:bound_lipschitz}
	\Big\vert \E \!\Big[ \ell_i\big( \tilde{\phi}_\nu^\star(\bx),y \big) \Big] -
		\E \!\Big[ \ell_i\big( f_{\tilde{\btheta}^\star_\nu}(\bx),y \big) \Big]
	 \Big\vert
	\\
	{}\leq
	\E \!\Big[ \big\vert
		\ell_i\big( \tilde{\phi}_\nu^\star(\bx),y \big) - \ell_i\big( f_{\tilde{\btheta}^\star_\nu}(\bx),y \big)
	\big\vert \Big]
	\\
	{}\leq M \E \!\Big[ \big\vert  f_{\tilde{\btheta}^\star_\nu}(\bx) - \tilde{\phi}_\nu^\star(\bx) \big\vert \Big]
		\leq M \nu
		\text{.}
\end{multline}
Since~$\tilde{\phi}_\nu^\star(\bx)$ is feasible for the perturbed~\eqref{P:csl_perturbed}, \eqref{E:bound_lipschitz} implies that~$\tilde{\btheta}^\star_\nu$ is feasible for~\eqref{P:csl}. By optimality, $P^\star \leq \E \!\big[ \ell_0\big( f_{\tilde{\btheta}^\star_\nu}(\bx),y \big) \big]$. Going back to~\eqref{E:lower_bound2}, we conclude that
\begin{align*}
	D^\star	&\geq \E \!\Big[ \ell_0\big( \tilde{\phi}_\nu^\star(\bx),y \big) \Big] - M\nu \norm{\bmut}_1
	\\
	{}&\geq P^\star + \E \!\Big[ \ell_0\big( \tilde{\phi}_\nu^\star(\bx),y \big)
		- \ell_0\big( f_{\tilde{\btheta}^\star_\nu}(\bx),y \big) \Big]
		- M\nu \norm{\bmut}_1
	\\
	{}&\geq P^\star - (1 + \norm{\bmut}_1) M \nu
		\text{,}
\end{align*}
where the last inequality stems from~\eqref{E:bound_lipschitz}.
\hfill$\blacksquare$


\section{Proof of Proposition~\ref{T:empirical}: The Estimation Gap}
\label{X:empirical}

\vspace{0.5\baselineskip}\noindent
\textbf{Feasibility.}
The proof follows by first showing that~$\bhtheta$ must be feasible for~\eqref{P:ecrm} using the same argument as in Appendix~\ref{X:param}. Then, leveraging the fact that~$\calH$ is PAC learnable~(Assumption~\ref{A:parametrization}), we can apply generalization bounds from classical learning theory.

Again by contradiction, suppose that all elements of~$\hat{\Theta}^\dagger(\bhmu)$ are infeasible for~\eqref{P:ecrm}. Then, for all~$\bhdtheta \in \Theta^\dagger(\bhmu)$ there exists~$i$ such that
\begin{equation*}
	\frac{1}{N_i} \sum_{n_i = 1}^{N_i} \ell_i\big( f_{\bhdtheta}(\bx_{n_i}), y_{n_i} \big) > c_i
		\text{,}
\end{equation*}
for~$(\bx_{n_i},y_{n_i}) \sim \fkD_i$. Then, from Lemma~\ref{T:danskin}, $\bzero \notin \partial d(\bhmu)$, which contradicts the optimality of~$\bhmu$. Hence, there must be~$\bhdtheta \in \hat{\Theta}^\dagger(\bhmu)$ feasible for~\eqref{P:ecrm}.

From the uniform bound in Assumption~\ref{A:empirical}, it holds with probability~$1-\delta$ over the data that
\begin{equation}\label{E:feasbility_vc}
\begin{aligned}
	\E_{(\bx,y) \sim \fkD_i} \!\Big[ \ell_i\big( f_{\btheta}(\bx),y \big) \Big] &\leq
		\frac{1}{N_i} \sum_{n_i = 1}^{N_i} \ell_i\big( f_{\btheta}(\bx_{n_i}), y_{n_i} \big)
		+ \zeta_i(N_i)
\end{aligned}
\end{equation}
for each~$i = 1,\dots,m$.

\vspace{0.5\baselineskip}\noindent
\textbf{Near-optimality.}
Let~$\bmu^\star$ and~$\bhmu$ be solutions of~\eqref{P:param_dual_csl} and~\eqref{P:empirical_dual} respectively and consider the set of dual minimizers
\begin{equation*}
	\Theta^\dagger(\bmu) = \argmin_{\theta \in \Theta} L(\btheta,\bmu)
	\quad\text{and}\quad
	\hat{\Theta}^\dagger(\hat{\bmu}) = \argmin_{\theta \in \Theta} \hat{L}(\btheta,\hat{\bmu})
\end{equation*}
for the Lagrangians defined in~\eqref{E:param_dual_function_csl} and~\eqref{E:empirical_lagrangian} respectively. Using the optimality~$\bmu^\star$, it holds that
\begin{align*}
	D^\star - \hat{D}^\star &= \min_{\btheta \in \Theta} L(\btheta,\bmu^\star)
		- \min_{\btheta \in \Theta} \hat{L}(\btheta,\bhmu)
	\\
	{}&\leq \min_{\btheta \in \Theta} L(\btheta,\bmu^\star) - \min_{\btheta \in \Theta} \hat{L}(\btheta,\bmu^\star)
		\text{.}
\end{align*}
Since~$\bhdtheta \in \hat{\Theta}^\dagger(\bmu^\star)$ is suboptimal for~$L(\btheta,\bmu^\star)$, we get
\begin{equation}\label{E:empirical_upper_bound}
	D^\star - \hat{D}^\star \leq L(\bhdtheta,\bmu^\star) - \hat{L}(\bhdtheta,\bmu^\star)
		\text{.}
\end{equation}
Using a similar argument yields
\begin{equation}\label{E:empirical_lower_bound}
	D^\star - \hat{D}^\star \geq L(\btheta^\dagger,\bhmu) - \hat{L}(\btheta^\dagger,\bhmu)
\end{equation}
for~$\btheta^\dagger \in \Theta^\dagger(\bhmu)$. Thus, we obtain that
\begin{multline}\label{E:gap_bound}
	\abs{D^\star - \hat{D}^\star} \leq
	\max \bigg\{
		\abs{L(\bhdtheta,\bmu^\star) - \hat{L}(\bhdtheta,\bmu^\star)},
	\\
		\abs{L(\btheta^\dagger,\bhmu) - \hat{L}(\btheta^\dagger,\bhmu)}
	\bigg\}
\end{multline}
Using the empirical bound from Assumption~\ref{A:empirical}, we obtain that
\begin{align}\label{E:vc_F_bound}
	\abs{L(\btheta,\bmu) - \hat{L}(\btheta,\bmu)} &\leq \zeta_0(N_0) + \sum_{i = 1}^m \mu_i \zeta_i(N_i)
	\notag\\
	{}&\leq (1 + \norm{\bmu}_1) \max_i \zeta_i(N_i)
		\text{,}
\end{align}
holds uniformly over~$\btheta$ with probability~$1-(m+1)\delta$. We omit the dependency of~$\zeta_i$ on~$\delta$ for conciseness.

\vspace{0.5\baselineskip}\noindent
\textbf{Union bound.}
To conclude, we use the union bound to combine~\eqref{E:feasbility_vc} for~$i=1,\dots,m$ with~\eqref{E:gap_bound} and~\eqref{E:vc_F_bound}. Doing so, we obtain that
\begin{align*}
	\abs{D^\star - \hat{D}^\star} &\leq (1 + \bar{\Delta}) \overline{\zeta}
	\\
	\E_{(\bx,y) \sim \fkD_i} \!\Big[ \ell_i\big( f_{\bhtheta}(\bx),y \big) \Big] &\leq
		c_i + \zeta_i(N_i)
\end{align*}
occur simultaneously with probability at least~$1-(3m+2)\delta$ for~$\overline{\zeta} = \max_i \zeta_i(N_i)$ and~$\bar{\Delta} = \max (\norm{\bmu^\star}_1, \norm{\bhmu}_1)$.
\hfill$\blacksquare$


\section{Proof of Theorem~\ref{T:convergence_quant}}
\label{X:convergence_quant}

\vspace{0.5\baselineskip}\noindent
\textbf{Deterministic duality gap.}
We proceed by proving that, for all~$\beta > 0$,
\begin{equation}\label{E:convergence_quant1}
	\hat{D}^\star - \rho - \eta\frac{mB^2}{2} - \beta
		\leq \hat{L}\big( \btheta^{(T)},\bmu^{(T)} \big)
		\leq \hat{D}^\star + \rho
		\text{,}
\end{equation}
from which we obtain~\eqref{E:convergence_quant} by using the near-PACC bound on~$\hat{D}^\star$ from Theorem~\ref{T:main} and choosing~$\beta = M\nu$ and~$\eta$ as in the statement of the theorem. Start by noticing from the definition of the empirical dual function~$\hat{d}$ in~\eqref{E:empirical_dual_function}, that the upper bound in~\eqref{E:convergence_quant1} holds trivially from the fact that~$\hat{d}(\bmu) \leq \hat{D}^\star$ for all~$\bmu \in \setR^m_+$. Thus, from the fact that~$\btheta^{(t)}$ is an approximate minimizer of the empirical Lagrangian~$\hat{L}$~(step~3 of Algorithm~\ref{L:primal_dual}), we obtain that
\begin{equation}\label{E:convergence_bounded}
	\hat{L}\big( \btheta^{(t)},\bmu^{(t)} \big)	\leq \hat{D}^\star + \rho
		\text{,}\quad \text{for all } t \geq 0
		\text{.}
\end{equation}

For the lower bound, we rely on the following relaxation of Danskin's classical theorem~\cite[Ch.~3]{Bertsekas15c}:

\begin{lemma}\label{T:subgrad_approx}

Let~$\btheta^\dagger$ be a $\rho$-\emph{approximate} minimizer of the empirical Lagrangian~\eqref{E:empirical_lagrangian} at~$\bmu$, i.e.,
\begin{equation}\label{E:approximate_minimizer}
	\hat{L}\big( \btheta^\dagger, \bmu\big) \leq \min_{\btheta \in \setR^p} \hat{L} \big(\btheta, \bmu \big) + \rho
		\text{,}
\end{equation}
for~$\rho \geq 0$. Then, the constraint slacks evaluate at~$\btheta^\dagger$
\begin{equation}\label{E:slacks}
	s_i(\btheta^\dagger) = \frac{1}{N_i} \sum_{n_i = 1}^{N_i}
		\ell_i\big( f_{\btheta^\dagger}(\bx_{n_i}), y_{n_i} \big) - c_i
		\text{, } i = 1,\dots,m
		\text{,}
\end{equation}
are \emph{approximate} supergradients of the dual function~\eqref{E:empirical_dual_function}. Explicitly,
\begin{equation}\label{E:subgrad_approx}
	\hat{d}(\bmu) \geq \hat{d}(\bmu^\prime)	+ \sum_{i = 1}^m (\mu_i - \mu_i^\prime) s_i(\btheta^\dagger)
		- \rho
\end{equation}
for all~$\bmu^\prime \in \setR^m_+$.
\end{lemma}

\begin{proof}

Recalling from~\eqref{E:empirical_dual_function} that~$\hat{d}(\bmu) = \min_{\btheta \in \setR^p} \hat{L} \big(\btheta, \bmu \big)$, we obtain from~\eqref{E:approximate_minimizer} that
\begin{equation}\label{E:subgrad_ineq1}
	\hat{d}(\bmu^\prime) \leq \hat{d}(\bmu^\prime) + \hat{d}(\bmu) - \hat{L} \big( \btheta^\dagger,	\bmu\big) + \rho
		\text{.}
\end{equation}
Additionally, we can upper bound~\eqref{E:subgrad_ineq1} by noticing that~$\hat{d}(\bmu^\prime) \leq \hat{L}(\btheta^\dagger,\bmu^\prime)$ for the suboptimal~$\btheta^\dagger$, yielding
\begin{equation}\label{E:subgrad_ineq2}
	\hat{d}(\bmu^\prime) \leq \hat{L} \big( \btheta^\dagger,\bmu^\prime \big) + \hat{d}(\bmu) - \hat{L} \big( \btheta^\dagger, \bmu \big)
			+ \rho
		\text{.}
\end{equation}
From~\eqref{E:empirical_lagrangian}, notice that the first term of the Lagrangians in~\eqref{E:subgrad_ineq2} are identical and cancel out, leading to~\eqref{E:subgrad_approx}.
\end{proof}

To proceed, let~$\calM^\star$ be the set of solutions of the dual problem~\eqref{P:empirical_dual}, i.e.,
\begin{equation}\label{E:empirical_lagrange_multipliers}
	\calM^\star = \argmax_{\bmu \in \setR^m_+}\ \hat{d}(\bmu)
		\text{.}
\end{equation}
We show next that for at least~$O(1/\beta)$ steps, the distance
\begin{equation}\label{E:lyapunov}
	U_t = \inf_{\bmu^\star \in \calM^\star} \big\Vert \bmu^{(t)} - \bmu^\star \big\Vert^2
\end{equation}
decreases by more than~$O(\beta)$. To do so, it is convenient to collect the constraint slacks from step~4 of Algorithm~\ref{L:primal_dual} into a vector~$\bs^{(t)} = \big[ s_i^{(t)} \big]_{i=1,\dots,m}$. Then, using the update in Algorithm~\ref{L:primal_dual}~(step~5), we write~\eqref{E:lyapunov} as
\begin{equation*}
	U_t = \inf_{\bmu^\star \in \calM^\star} \norm{\Big[ \bmu^{(t-1)} + \eta \bs^{(t-1)} \Big]_+ - \bmu^\star}^2
		\text{.}
\end{equation*}
Since~$\calM^\star \subset \setR^m_+$, we can use the non-expansiveness of the projection~$[\cdot]_+$~\cite{Bertsekas09c} to obtain
\begin{equation}\label{E:algorithm_update}
	U_t \leq \inf_{\bmu^\star \in \calM^\star} \norm{\bmu^{(t-1)} + \eta \bs^{(t-1)} - \bmu^\star}^2
		\text{.}
\end{equation}

To proceed, expand the norms in~\eqref{E:algorithm_update} to get
\begin{align*}
	U_t &\leq \inf_{\bmu^\star \in \calM^\star} \Big[
		\big\Vert\bmu^{(t-1)} - \bmu^\star \big\Vert^2
			+ 2 \eta \left( \bmu^{(t-1)} - \bmu^\star \right)^\top \bs^{(t-1)}
	\Big]
	\\
	{}&+ \eta^2 \big\Vert\bs^{(t-1)}\big\Vert^2
		\text{.}
\end{align*}
Using Lemma~\ref{T:subgrad_approx} and the fact that the~$\ell_i$ are bounded, we then obtain
\begin{equation*}
	U_t \leq U_{t-1} + 2 \eta \Big[ \hat{d}\big( \bmu^{(t-1)} \big) - \hat{D}^\star + \rho \Big] + m \eta^2 B^2
		\text{.}
\end{equation*}
Since~$\hat{D}^\star = \hat{d}( \bmu^\star )$ for all~$\bmu^\star \in \calM^\star$, the second term no longer depends on the choice of~$\bmu^\star$, so that the infimum applies only to the distance between~$\bmu^{(t-1)}$ and~$\bmu^\star$, which we write as~$U_{t-1}$ using the definition in~\eqref{E:lyapunov}. Solving this recursion yields
\begin{equation}\label{E:algorithm_update3}
	U_t \leq U_0 + 2 \eta \sum_{t = 0}^{t-1} \Delta_t
		\text{,}
\end{equation}
for
\begin{equation}\label{E:delta_t}
	\Delta_t = \hat{d}\big( \bmu^{(t)} \big) - \hat{D}^\star + \rho + \eta\frac{m B^2}{2}
		\text{.}
\end{equation}

To conclude, notice that~$\hat{d}(\bmu) \leq \hat{D}^\star$ for all~$\bmu \in \setR^m_+$. Hence, when~$\bmu^{(t)}$ is sufficiently far from the optimum and the step size~$\eta$ is sufficiently small, we have~$\Delta_t < 0$ and~\eqref{E:algorithm_update3} shows that the distance to the optimum~$U_t$ decreases. Formally, fix a precision~$\beta > 0$ and let~$(T-1) = \min \{ t \mid \Delta_t > -\beta\}$. Then, from the definition of~$\Delta_t$ we obtain the desired lower bound in~\eqref{E:convergence_quant1} by noting that
\begin{equation*}
	\Delta_{T-1} > -\beta \Leftrightarrow \hat{d}\big( \bmu^{(T-1)} \big) > \hat{D}^\star - \rho - \eta\frac{m B^2}{2} - \beta
\end{equation*}
and that~$\hat{d}\big( \bmu^{(T-1)} \big) \leq \hat{L}(\btheta,\bmu^{(T-1)})$. What is more, \eqref{E:algorithm_update3} yields
\begin{equation*}
	T \leq \frac{U_{0}}{2 \eta \beta} + 1 = O\big( \beta^{-1} \big)
		\text{.}
\end{equation*}


\section{Proof of Theorem~\ref{T:convergence_rand}}
\label{X:convergence_rand}

We will proceed by showing that the randomized solution is feasible and near-optimal for the empirical problem~\eqref{P:ecrm}. We can then leverage Assumption~\ref{A:empirical} and Propositions~\ref{T:param} and~\ref{T:empirical} and apply the union bound to obtain the result with respect to the statistical losses in~\eqref{E:convergence_rand}. We start by proving~\eqref{E:rand_feasibility} and then proceed with~\eqref{E:rand_optimality}.

\vspace{0.5\baselineskip}\noindent
\textbf{Randomized feasibility.}
We begin with a technical lemma showing that the~$\bmu^{(t)}$ generated by Algorithm~\ref{L:primal_dual} do not move too far away from the set of Lagrange multipliers.

\begin{lemma}\label{T:iterates_norm_bound}
	Let~$\calM^\star$ denote the set of solutions of the dual problem~\eqref{P:empirical_dual}, i.e., $\calM^\star = \argmax_{\bmu \in \setR^m_+}\ \hat{d}(\bmu)$, and~$\{\bmu^{(t)}\}$ denote the sequence of dual variables generated by Algorithm~\ref{L:primal_dual}. Then, under Assumption~\ref{A:slater}, there exists a constant~$C < \infty$ such that
	\begin{equation}\label{E:algorithm_mu_bound}
		\inf_{\bhmu \in \calM^\star} \big\Vert \bmu^{(t)} - \bhmu \big\Vert \leq C
			\text{.}
	\end{equation}
\end{lemma}

\begin{proof}
Start by defining the set of approximate Lagrange multipliers
\begin{equation}\label{eqn_bounded_lambda}
    \calD \triangleq \left\{\bmu \in \setR^m_+ \mid \hat{D}^\star - \hat{d}(\bmu) \leq \rho + \eta \frac{m B^2}{2} \right\}
    	\text{.}
\end{equation}
This is set is not empty since there exists at least one~$\bmu$ that achieves~$\hat{D}^\star$. Indeed, while~\eqref{P:empirical_dual} optimizes over the open set~$\setR^m_+$, Lemma~\ref{T:norm_mu_bound} shows that the existence of a strictly feasible point means it is equivalent to a problem over a compact set~($\norm{\bmu^\star}_1$ is bounded). The remainder of the proof is separated in the cases~$\bmu \in \calD$ and~$\bmu \notin \calD$.

For~$\bmu \in \calD$, note that the same argument used to prove Lemma~\ref{T:norm_mu_bound} yields~$\norm{\bmu}_1 \leq (B - \hat{D}^\star)/\xi \triangleq C$. Since~$\cal\calM^\star \subseteq \calD$, it holds that
\begin{equation}\label{E:bound_in_D}
	\inf_{\bhmu \in \calM^\star} \norm{\bmu - \bhmu} \leq C
		\text{.}
\end{equation}

For the~$\bmu \notin \calD$ case, denote by~$\calB_{\eta B}$ the~$l_2$-ball centered at the origin with radius~$\eta B$ and by~$\oplus$ the Minkowski sum. Observe that if~$\bmu^{(t-1)} \in \calD$, then~$\bmu^{(t)} \in \calD \oplus \calB_{\eta B}$. This is due to the fact that, since the losses are bounded,~$s^{(t-1)}_i \leq B$ in step~5 of Algorithm~\ref{L:primal_dual}. This means that once the iterates step into~$\calD$, they must have bounded norm once they step outside. Hence, we can consider without loss of generality the case in which~$\bmu^{(0)} \notin \calD$. Indeed, after~$T_0 = \min \{t \mid \bmu^{(t)} \in \calD \}$, the first iterate \emph{outside}~$\calD$ will have bounded norm and the study of the sequence reduces to the case in which~$\bmu^{(0)} \notin \calD$.

To proceed, let us therefore study the sequence~$\bmu^{(t \wedge T_0)}$, where~$a \wedge b$ denotes the minimum between~$a$ and~$b$. Notice from~\eqref{E:delta_t} that~$\Delta_t < 0$ for any~$\bmu \notin \calD$, so that~\eqref{E:algorithm_update3} implies
\begin{equation}\label{E:bound_in_Dc}
	\inf_{\bhmu \in \calM^\star} \big\Vert \bmu^{(t \wedge T_0)} - \bhmu \big\Vert < \inf_{\bhmu \in \calM^\star} \big\Vert \bmu^{(0)} - \bhmu \big\Vert
		\text{.}
\end{equation}
Given that~$\bmu^{(0)} = \zeros$ in Algorithm~\ref{L:primal_dual}, the right-hand side reduces to~$\inf_{\bhmu \in \calM^\star} \big\Vert \bhmu \big\Vert \leq C$, since~$\calM^\star \subseteq \calD$. Combining~\eqref{E:bound_in_D} with~\eqref{E:bound_in_Dc} yields the desired result.
\end{proof}

Let us now show that for~$i=1,\dots,m$ it holds that
\begin{equation}\label{E:rand_feasibility_empirical}
	\E_{\btheta \sim \fke_{T}} \left[ \frac{1}{N_i} \sum_{n_i = 1}^{N_i}
		\ell_i\big( f_{\btheta}(\bx_{n_i}), y_{n_i} \big) - c_i \right] \leq \frac{2 C}{\eta T}
		\text{,}
\end{equation}
which reduces to
\begin{equation}
	\frac{1}{T} \sum_{t = 0}^{T-1} s_i^{(t)} \leq \frac{2 C}{\eta T}
		\text{,}\quad i=1,\dots,m
		\text{,}
\end{equation}
for~$s_i^{(t)}$ defined in step~4 of Algorithm~\ref{L:primal_dual} and the constant~$C$ from Lemma~\ref{T:iterates_norm_bound}. To do so, start by noticing from step~5 of Algorithm~\ref{L:primal_dual} that for~$i = 1,\dots,m$ it holds that
\begin{equation*}
	\mu_i^{(t)} \geq \mu_i^{(t-1)} + \eta s_i^{(t-1)}
		\text{,} \quad \text{for all } t
		\text{.}
\end{equation*}
Solving the recursion and recalling that~$\bmu^{(0)} = \zeros$ then yields
\begin{equation*}
	\mu_i^{(t)} \geq \eta \sum_{j = 0}^{t-1} s_i^{(j)}
		\text{,} \quad i = 1,\dots,m
		\text{.}
\end{equation*}
Hence,
\begin{equation*}
	\frac{1}{T} \sum_{t = 0}^{T-1} s_i^{(t)} \leq \frac{\mu_i^{(t)}}{\eta T}
		\text{,} \quad i = 1,\dots,m
		\text{.}
\end{equation*}
Since~$\mu_i^{(t)} \geq 0$, it holds for any~$\bmu^\star \in \calM^\star$ that
\begin{equation*}
	\frac{1}{T} \sum_{t = 0}^{T-1} s_i^{(t)} \leq \frac{\big\vert \mu_i^{(T)} - \hat{\mu}_i^\star \big\vert + \big\vert \hat{\mu}_i^\star \big\vert}{\eta T}
		\leq \frac{\big\Vert \bmu^{(T)} - \bhmu \big\Vert + \big\Vert \bhmu \big\Vert}{\eta T}
		\text{.}
\end{equation*}
Applying Lemma~\ref{T:iterates_norm_bound} yields~\eqref{E:rand_feasibility_empirical}. Applying the uniform empirical bound from Assumption~\ref{A:empirical} yields that for each~$i = 1,\dots,m$, with probability at least~$1-\delta$,
\begin{equation}\label{E:rand_feasbility_isolated}
	\E_{\substack{(\bx,y) \sim \fkD_i \\ \btheta \sim \fke_{T}}} \!\big[ \ell_i\big( f_{\btheta}(\bx), y \big) \big]
		\leq c_i + \frac{2 C}{\eta T} + \zeta_i(N_i)
		\text{.}
\end{equation}

\vspace{0.5\baselineskip}\noindent
\textbf{Randomized optimality.}
As in the feasibility proof, we begin with a lemma bounding the ergodic complementary slackness for the joint sequence~$\{(s_i^{(t)}, \bmu^{(t)})\}$.

\begin{lemma}\label{T:rand_complementary_slackness}
	Consider sequence~$\{(s_i^{(t)},\bmu^{(t)})\}$ for~$t = 0,\dots,T-1$ obtained from Algorithm~\ref{L:primal_dual}. It holds that
	\begin{equation}\label{E:complementary_slackness}
		\frac{1}{T} \sum_{t = 0}^{T=1} \left[ \sum_{i = 1}^{m} \mu_i^{(t)} s_i^{(t)} \right] \geq -\eta \frac{m B^2}{2}
			\text{.}
	\end{equation}
\end{lemma}

\begin{proof}
To simplify the notation, define the vector~$\bs^{(t)} = [s_i^{(t)}]$ that collects the slacks from step~4 of Algorithm~\ref{L:primal_dual}. Then, the complementary slackness can be written as
\begin{equation*}
	\sum_{i = 1}^{m} \mu_i^{(t)} s_i^{(t)} = \bmu^{(t)\top} \bs^{(t)}
		\text{,}
\end{equation*}
where~${}^\top$ denotes the transposition operation. To bound~\eqref{E:complementary_slackness}, we once again use the update in step~5 of Algorithm~\ref{L:primal_dual} together with the non-expansiveness of the projection to obtain
\begin{align*}
	\big\Vert \bmu^{(t)} \big\Vert^2 &\leq \norm{\bmu^{(t-1)} + \eta \bs^{(t-1)}}^2
	\\
	{}&= \big\Vert \bmu^{(t-1)} \big\Vert^2 + \eta^2 \big\Vert \bs^{(t-1)} \big\Vert^2 + 2 \eta \bmu^{(t-1)\top} \bs^{(t-1)}
\end{align*}
Applying this relation recursively from~$T$ and using the fact that the losses are bounded~(and, thus, so are the constraint slacks~$s_i$) and~$\bmu^{(0)} = \zeros$, we obtain
\begin{equation}\label{E:rand_complementary_slackness1}
	\big\Vert \bmu^{(T)} \big\Vert^2 \leq T \eta^2 (m B^2) + 2 \eta \sum_{j = 0}^{T-1} \bmu^{(j)\top} \bs^{(j)}
\end{equation}
Noticing that~$\big\Vert \bmu^{(T)} \big\Vert^2 \geq 0$ and dividing~\eqref{E:rand_complementary_slackness1} by~$T$ yields the desired result.
\end{proof}

With Lemma~\ref{T:rand_complementary_slackness} in hand, we can now proceed with bounding the suboptimality of the randomized solution. To do so, we once again start by bounding the empirical objective. Explicitly, let~$\hat{F}_i(\btheta) \triangleq \frac{1}{N_i} \sum_{n_i = 1}^{N_i} \ell_0\big( f_{\btheta}(\bx_{n_i}), y_{n_i} \big)$ and notice that
\begin{equation*}
	\E_{\fke_{T}} \!\left[ \hat{F}_0(\btheta) \right]
		= \frac{1}{T} \sum_{t = 0}^{T-1} \hat{F}_0(\btheta)
		\text{.}
\end{equation*}
By adding and subtracting~$\sum_{i = 1}^m \mu_i^{(t)} s_i^{(t)}$ and recognizing the expression for the empirical Lagrangian~\eqref{E:empirical_lagrangian}, we then obtain
\begin{equation*}
	\E_{\fke_{T}} \!\left[ \hat{F}_0(\btheta) \right] = \frac{1}{T} \sum_{t = 0}^{T-1} \hat{L}\big( \btheta^{(t)},\bmu^{(t)} \big) - \sum_{i = 1}^m \mu_i^{(t)} s_i^{(t)}
		\text{.}
\end{equation*}
Using Lemma~\ref{T:rand_complementary_slackness} and the approximate minimizer property of~$\btheta^{(t)}$~(step~3 of Algorithm~\ref{L:primal_dual}), we then obtain the inequality
\begin{equation}\label{E:rand_optimality1}
	\E_{\fke_{T}} \!\left[ \hat{F}_0(\btheta) \right] \leq \frac{1}{T} \sum_{t = 0}^{T-1}
		\left[ \min_{\btheta \in \Theta}\ \hat{L}\big(\btheta,\bmu^{(t)}\big) + \rho \right]
			+ \eta \frac{m B^2}{2}
		\text{.}
\end{equation}
Observe that the minimum in~\eqref{E:rand_optimality1} is the empirical dual function~\eqref{E:empirical_dual_function} and by the definition of the empirical dual problem~\eqref{P:empirical_dual} it holds that
\begin{equation}\label{E:rand_optimality2}
	\E_{\fke_{T}} \!\left[ \hat{F}_0(\btheta) \right] \leq \hat{D}^\star + \rho + \eta \frac{m B^2}{2}
		\text{.}
\end{equation}

To conclude the proof, use Propositions~\ref{T:param} and~\ref{T:empirical} to bound~$\hat{D}^\star$ in~\eqref{E:rand_optimality2} and get that with probability~$1-2(m+1)\delta$
\begin{equation}
	\E_{\fke_{T}} \!\left[ \hat{F}_0(\btheta) \right] \leq P^\star + \rho + \eta \frac{m B^2}{2} + (1+\Delta) (M\nu + \bar{\zeta})
		\text{.}
\end{equation}
Once again using the uniform empirical bound from Assumption~\ref{A:empirical} yields that, with probability at least~$1-(2m+3)\delta$,
\begin{equation}\label{E:rand_optimality_isolated}
\begin{aligned}
	\E_{\substack{(\bx,y) \sim \fkD_i \\ \btheta \sim \fke_{T}}} \!\big[ \ell_0\big( f_{\btheta}(\bx), y \big) \big]
		&\leq P^\star + \rho + \eta \frac{m B^2}{2}
		\\
		{}&+ (1+\Delta) M\nu + (2+\Delta) \bar{\zeta}
		\text{.}
\end{aligned}
\end{equation}

\vspace{0.5\baselineskip}\noindent
\textbf{Union bound.}
Using the union bound to combine~\eqref{E:rand_feasbility_isolated} and~\eqref{E:rand_optimality_isolated} and the choice of step size in~\eqref{E:eta} concludes the proof.
\hfill$\blacksquare$

\end{document}